%% file: main.tex
\documentclass[sigconf]{aamas}

\usepackage{balance} %

\setcopyright{ifaamas}
\acmConference[AAMAS '24]{Proc.\@ of the 23rd International Conference
on Autonomous Agents and Multiagent Systems (AAMAS 2024)}{May 6 -- 10, 2024}
{Auckland, New Zealand}{N.~Alechina, V.~Dignum, M.~Dastani, J.S.~Sichman (eds.)}
\copyrightyear{2024}
\acmYear{2024}
\acmDOI{}
\acmPrice{}
\acmISBN{}

\usepackage{balance}

\usepackage{algorithm}
\usepackage{algorithmic}

\usepackage{xcolor}
\usepackage{subfigure}
\usepackage{svg}
\usepackage{mathtools}
\usepackage{adjustbox}

\usepackage{dsfont}
\usepackage{float}
\theoremstyle{plain}
\newtheorem{theorem}{Theorem}[section]
\newtheorem*{theorem*}{Theorem}

\newtheorem{lemma}[theorem]{Lemma}

\theoremstyle{definition}
\newtheorem{definition}[theorem]{Definition}

\theoremstyle{remark}

\usepackage{svg}

\acmSubmissionID{fp0488}

\title[MESA]{MESA: Cooperative Meta-Exploration in Multi-Agent Learning through Exploiting State-Action Space Structure}

\author{Zhicheng Zhang}
\affiliation{
  \institution{Carnegie Mellon University}
  \city{Pittsburgh, Pennsylvania}
  \country{United States}}
\email{zhichen3@cs.cmu.edu}
\authornote{Equal contribution.}

\author{Yancheng Liang}
\affiliation{
  \institution{University of Washington}
  \city{Seattle, Washington}
  \country{United States}}
\email{yancheng@cs.washington.edu}
\authornotemark[1]

\author{Yi Wu}
\affiliation{
  \institution{Tsinghua University}
  \city{Beijing}
  \country{China}}
\email{jxwuyi@gmail.com}

\author{Fei Fang}
\affiliation{
  \institution{Carnegie Mellon University}
  \city{Pittsburgh, Pennsylvania}
  \country{United States}}
\email{feif@cs.cmu.edu}

\begin{abstract}
  Multi-agent reinforcement learning (MARL) algorithms often struggle to find strategies close to Pareto optimal Nash Equilibrium, owing largely to the lack of efficient exploration. The problem is exacerbated in sparse-reward settings, caused by the larger variance exhibited in policy learning. This paper introduces MESA, a novel meta-exploration method for cooperative multi-agent learning. It learns to explore by first identifying the agents' high-rewarding joint state-action subspace from training tasks and then learning a set of diverse exploration policies to ``cover'' the subspace. These trained exploration policies can be integrated with any off-policy MARL algorithm for {test-time} tasks. We first showcase MESA's advantage in a multi-step matrix game. Furthermore, experiments show that with learned exploration policies, MESA achieves significantly better performance in sparse-reward tasks in several multi-agent particle environments and multi-agent MuJoCo environments, and exhibits the ability to generalize to {more challenging tasks at test time.}
\end{abstract}


\newcommand{\BibTeX}{\rm B\kern-.05em{\sc i\kern-.025em b}\kern-.08em\TeX}


\begin{document}

\pagestyle{fancy}
\fancyhead{}

\maketitle 

\input{content/intro}

\input{content/related_work}

\input{content/preliminaries}

\input{content/example}

\input{content/methods}

\input{content/experiments}

\input{content/conclusions}

\begin{acks}
This research is supported in part by NSF IIS-2046640 (CAREER) and Sloan Research Fellowship. We thank NVIDIA for providing computing resources. Zhicheng Zhang is supported in part by SCS Dean’s Fellowship. The funders have no role in study design, data collection and analysis, decision to publish, or preparation of the manuscript.
\end{acks}

\bibliographystyle{ACM-Reference-Format}
\balance
\bibliography{sample}

\clearpage

\appendix

\input{content/appendix}

\end{document}

%% file: content/intro.tex
\section{Introduction}

Reinforcement learning (RL) algorithms often adopt a trial-and-error learning paradigm and optimize the policy based on the reward signals given by the environment. %
{The effectiveness of RL relies on efficient exploration, especially in sparse reward settings, as it is critical to get sufficient experiences with high rewards to guide the training.}

The exploration challenge has been studied extensively and existing works can be categorized mainly into two streams. One core idea with great success is to incentivize the agent to visit the under-explored states more frequently by adding an intrinsic reward based on a visitation measure \cite{bellemare2016unifying,pathak2017curiosity,ostrovski2017count,tang2017exploration} or some other heuristics~\cite{influence3:hughes2018inequity,influenceexploration}. 

However, in multi-agent settings, due to the exponential growth of the joint state-action space, simply visiting more novel states can be increasingly ineffective. Exploration policies need to better capture the low-dimensional structure of the tasks and leverage the structural knowledge for higher exploration efficiency.

\begin{figure}[!t]
    \centering
    \adjustbox{trim=0cm 0cm 0.5cm 0cm}{%
        \includegraphics[scale=0.093]{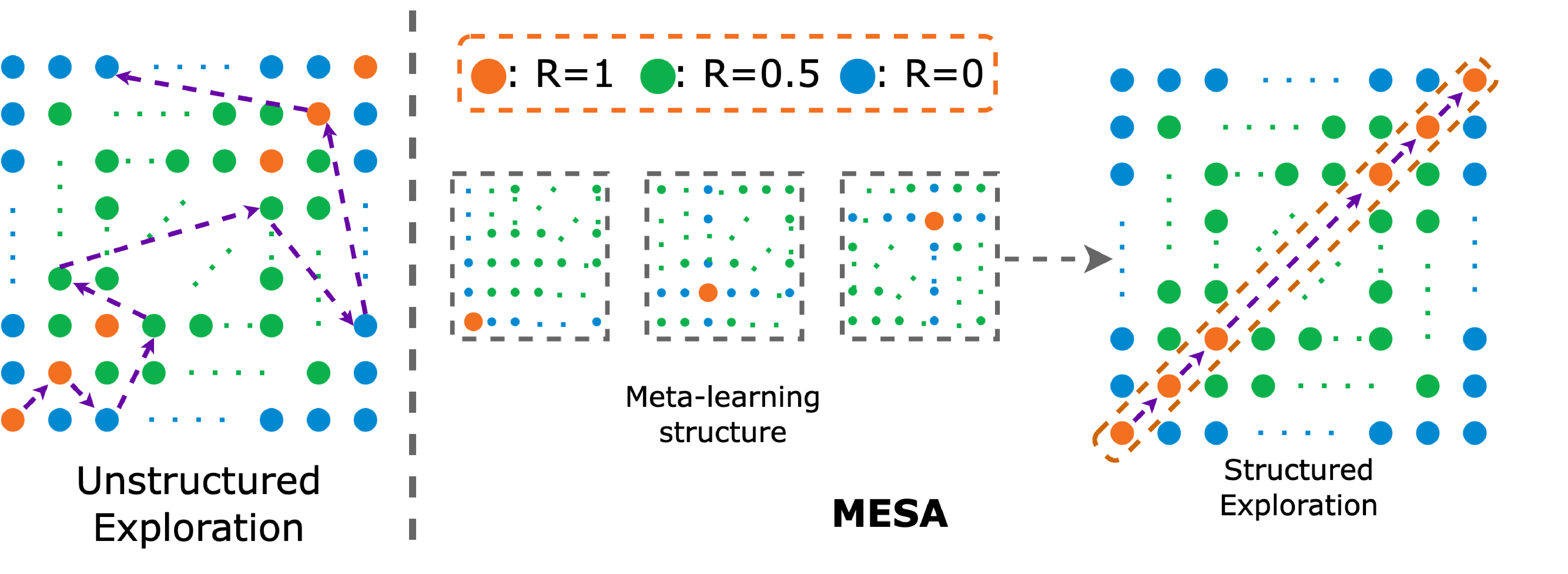}
    }
    \caption{Illustration of structured exploration and unstructured exploration behavior in the $2$-player climb game. The rows and columns indicate the players' action space. %
    While unstructured exploration aims to visit novel states, structured exploration exploits structures in the joint state-action space, helping agents coordinatedly and more efficiently explore the potential high-reward subspace.
    }
    \Description{Illustration of structured exploration and unstructured exploration behavior in the $2$-player climb game. The rows and columns indicate the players' action space. While unstructured exploration aims to visit novel states, structured exploration exploits structures in the joint state-action space, helping agents coordinatedly and more efficiently explore the potential high-reward subspace.}
    \label{fig:intro_illustrate}
\end{figure}

{Another line of work specifically learns exploration strategies. However, these works do not explicitly consider the underlying task structure.
For example, \citeauthor{MAVEN} conditions the policy on a shared latent variable~\cite{MAVEN} learned via mutual information maximization. 
\citeauthor{cooperativeexploration} adopts a goal-conditioned exploration strategy by setting state features as goals~\cite{cooperativeexploration}. 
Other works in the single-agent settings~\cite{GEP-PG,parisi2021interestingdeepak,schafer2022decoupled} learn exploration policies through a pre-defined intrinsic reward.
All these works train the exploration policy using task-agnostic exploration-specific rewards.}
{In Section~\ref{sec:example-game}, we will present a simple matrix game to show that popular exploration methods can have difficulties finding the optimal solution due to the reward structure of the game.}

How can we enable the agents to more effectively explore by leveraging the intrinsic structure of the environment?
We adopt a meta-exploration framework (i.e., learning to explore) for MARL: we first train multiple \emph{structured exploration policies} from a set of training tasks (referred to as the meta-training stage), {and then use these exploration policies to facilitate agents' learning in a test-time task, which is typically a new task sampled from the task distribution (referred to as meta-testing stage).}
We develop a multi-agent meta-exploration method,
\textit{Cooperative \textbf{M}eta-\textbf{E}xploration in \textbf{M}ulti-\textbf{A}gent Learning through Exploiting State-Action Space Structure} (MESA) for fully cooperative settings.
MESA leverages the task structures by explicitly identifying the agents' high-rewarding joint state-action subspace in the training tasks. It then trains a set of diverse exploration policies to cover this identified subspace. The exploration policies are trained with a reward scheme induced by the distance to the high-rewarding subspace. The meta-learned exploration policies can be combined with any off-policy MARL algorithm during the meta-testing stage by randomly selecting learned exploration policies to collect valuable experiences. Such structured exploration can help the agents to learn good joint policies efficiently (Figure \ref{fig:intro_illustrate}).
We empirically show the success of MESA on the matrix climb game and its harder multi-stage variant. In addition, we evaluate MESA in {two continuous control tasks, i.e., the MPE environment~\cite{MADDPG} and the multi-agent MuJoCo benchmark~\cite{peng2021facmac}. }%
We demonstrate the superior performance of MESA compared to existing multi-agent learning and exploration algorithms. 
{Furthermore, we show that MESA is capable of generalizing to unseen test-time tasks that are more challenging than any of the training tasks.} %

%% file: content/related_work.tex
\section{Related Work}

Exploration has been a long-standing challenge in RL with remarkable progress achieved in the single-agent setting~\cite{bellemare2016unifying,pathak2017curiosity,ostrovski2017count,tang2017exploration,RND,riedmiller2018learning,ecoffet2021first}. Most of these works maintain pseudo-counts over states and construct intrinsic rewards to encourage the agents to visit rarely visited states more frequently~\cite{bellemare2016unifying,pathak2017curiosity,ostrovski2017count,tang2017exploration}.
These count-based methods have been extended to the multi-agent setting by incentivizing intra-agent interactions or social influence~\cite{influence2:jaques2018intrinsic,influence3:hughes2018inequity,influence1:jaques2019social,influenceexploration}. However, in the multi-agent setting, a simple count-based method can be less effective due to the partial observability of each agent, an exponentially large joint state-action space, and the existence of multiple non-Pareto-optimal NE. Therefore, recent works focus on discovering the structures of possible multi-agent behaviors. For example, \cite{MAVEN} adopts variational inference to learning structured latent-space-policies; \cite{UNEVEN} generates similar tasks with simpler reward functions to promote cooperation; \cite{cooperativeexploration} learns to select a subset of state dimensions for efficient exploration.
We follow a meta-learning framework and learn structured exploration strategies by exploiting high-rewarding subspace in the joint state-action space.
Our method also leverages a count-based technique as a subroutine during the meta-training phase to prevent over-exploitation and mode collapse.

Meta reinforcement learning (meta-RL) is a popular RL paradigm that focuses on training a policy that can quickly adapt on an unseen task at test time~\cite{RL2:duan2016rl,MAML,MAESN,xu2018learning,lan2019meta,rakelly2019efficient,varibad,MetaCure}.
Such a paradigm has been extended to the setting of learning to explore. The key idea is to meta-learn a separate exploration policy that can be used in the testing task.
Most closely related to our work is \cite{parisi2021interestingdeepak}, where an exploration policy is pretrained on a set of training tasks. However, their method is designed for the single-agent setting and learns the exploration policy by using a task-agnostic intrinsic reward to incentivize visitation of interesting states
, while we directly utilize the task reward to learn the structure of the environments. Other existing works in meta-exploration propose to learn a latent-space exploration policy that is conditioned on a task variable, which can be accomplished by meta-policy gradient~\cite{xu2018learning,MAESN,lan2019meta}, variational inference~\cite{rakelly2019efficient} or information maximization~\cite{MetaCure} over the training tasks. Therefore, at test time, posterior inference can be performed for the latent variable towards fast exploration strategy adaption. 
Our approach follows a similar meta-exploration paradigm by learning additional exploration policies. However, existing meta-exploration methods focus on the single-agent setting while we consider much more challenging multi-agent games
{with a distribution of similarly-structured tasks, for example, the MPE environment~\cite{MADDPG} with a distribution of target landmarks that the agents need to reach.}
In addition, we meta-learn a discrete set of exploration policies through an iterative process, which results in a much simpler meta-testing phase without the need for posterior sampling or gradient updates on exploration policies. 
Besides, some other methods pretrain exploration policies from an offline dataset~\cite{singh2020parrot,offline-explore-2:dorfman2020offline,offline-explore-1:pong2021offline}, which is beyond the scope of this paper.

Finally, our approach largely differs from the setting of multi-task learning~\cite{mtrl-distill-1:parisotto2016actor,mtrl-HER-1:andrychowicz2017hindsight,mtrl-hier-2:andreas2017modular,mtrl-direct-2:espeholt2018impala,mtrl-direct-1:hessel2019multi}{, which are commonly evaluated in environments with heterogeneous tasks or scenarios}.
Our exploration policies are \emph{not} trained to achieve high returns in the training tasks. Instead, they are trained to reach as many high-reward \emph{state-action pairs} as possible collected in a diverse set of tasks. Therefore, the state-action pairs covered by a single exploration policy are very likely to be distributed across different training tasks.

%% file: content/preliminaries.tex
\section{Preliminaries}

\textbf{Dec-POMDP.} We consider fully-cooperative Markov games described by a decentralized partially observable Markov decision process (Dec-POMDP), which is defined by $\left\langle \mathcal{S}, \mathcal{A}, P, R, \Omega, \mathcal{O},n,\gamma \right\rangle$. $\mathcal{S}$ is the state space. $\mathcal{A} \equiv \mathcal{A}_1 \times ... \times \mathcal{A}_n$ is the joint action space. The dynamics is defined by the transition function $P(s'\mid s,\boldsymbol{a})$. Agents share a reward function $R(s,\boldsymbol{a})$, and $\gamma \in (0,1)$ is the discount factor. $\Omega \equiv \Omega_1 \times .. \times \Omega_n$ is the joint observation space, where $\Omega_i$ is the observation space for agent $i$. At each timestep, each agent $i$ only has access to its own observation $o_i \in \Omega_i$ defined by the function $\mathcal{O}: \mathcal{S}\times\mathcal{A}\mapsto\Omega$. The goal of agents in Dec-POMDP is to maximize the common expected discounted return under the joint policy $\boldsymbol{\pi}$: $\mathcal{J}(\boldsymbol{\pi}) =  \mathbb{E}_{\boldsymbol{\pi}}\left[\sum_{t} \gamma^t R(s_t,\boldsymbol{a}_t)\right]$.

\textbf{Learning to Explore.}
Meta-RL assumes a task distribution $p(\mathcal{T})$ over tasks, and an agent aims to learn to quickly adapt to a test-time task $\mathcal T_\text{test}$ drawn from $p(\mathcal{T})$ after training in a batch of training tasks $\{\mathcal{T}_i\mid\mathcal{T}_i \sim p(\mathcal{T})\}_{i=1}^B$.
Inspired by the explicit exploration methods \cite{GEP-PG,MetaCure}, we adopt a meta-exploration framework for MARL: we learn joint exploration policies $\boldsymbol{\pi}_e$ from training tasks $\{\mathcal{T}_i\mid\mathcal{T}_i \sim p(\mathcal{T})\}_{i=1}^B$ and use $\boldsymbol{\pi}_e$ to collect experiences for the training of the agents' policy profile $\boldsymbol{\pi}$ in task $\mathcal{T}_\text{test}$, denoted as $\boldsymbol{\pi}(\boldsymbol{\pi}_e, \mathcal{T}_\text{test})$. Formally, the objective of meta-exploration is
\begin{equation}
    \max_{\boldsymbol{\pi}_e} \mathbb{E}_{\mathcal{T}_\text{test}\sim p(\mathcal{T})}\left[  \mathbb{E}_{\boldsymbol{\pi}(\boldsymbol{\pi}_e, \mathcal{T}_\text{test})} \left[\sum_{t} \gamma^t R_i(s_t,\boldsymbol{a}_t)\right] \right].
\end{equation}

\textbf{Nash Equilibrium and Pareto Optimality.}
A joint policy $\boldsymbol{\pi}$ is an NE if each agent's policy $\pi_i$ is a best response to the other agents' policies $\boldsymbol{\pi}_{-i}$. That is, for any agent $i$'s alternative policy $\pi'_i$, we have $Q_i(\boldsymbol{\pi}) \geq Q_i(\pi'_i,\boldsymbol{\pi}_{-i})$, where $Q_i$ is the value function for agent $i$. 
A joint policy $\boldsymbol{\pi}$ is Pareto optimal if there does not exist an alternative joint policy $\boldsymbol{\pi}'$
such that $\forall i,\ Q_i(\boldsymbol{\pi}') \geq Q_i(\boldsymbol{\pi})$ and $\exists i,\ Q_i(\boldsymbol{\pi}') > Q_i(\boldsymbol{\pi})$. %

%% file: content/example.tex
\section{A Motivating Example: Climb Game}
\label{sec:example-game}

\begin{figure*}[ht]
    \centering
    \includegraphics[scale=0.45]{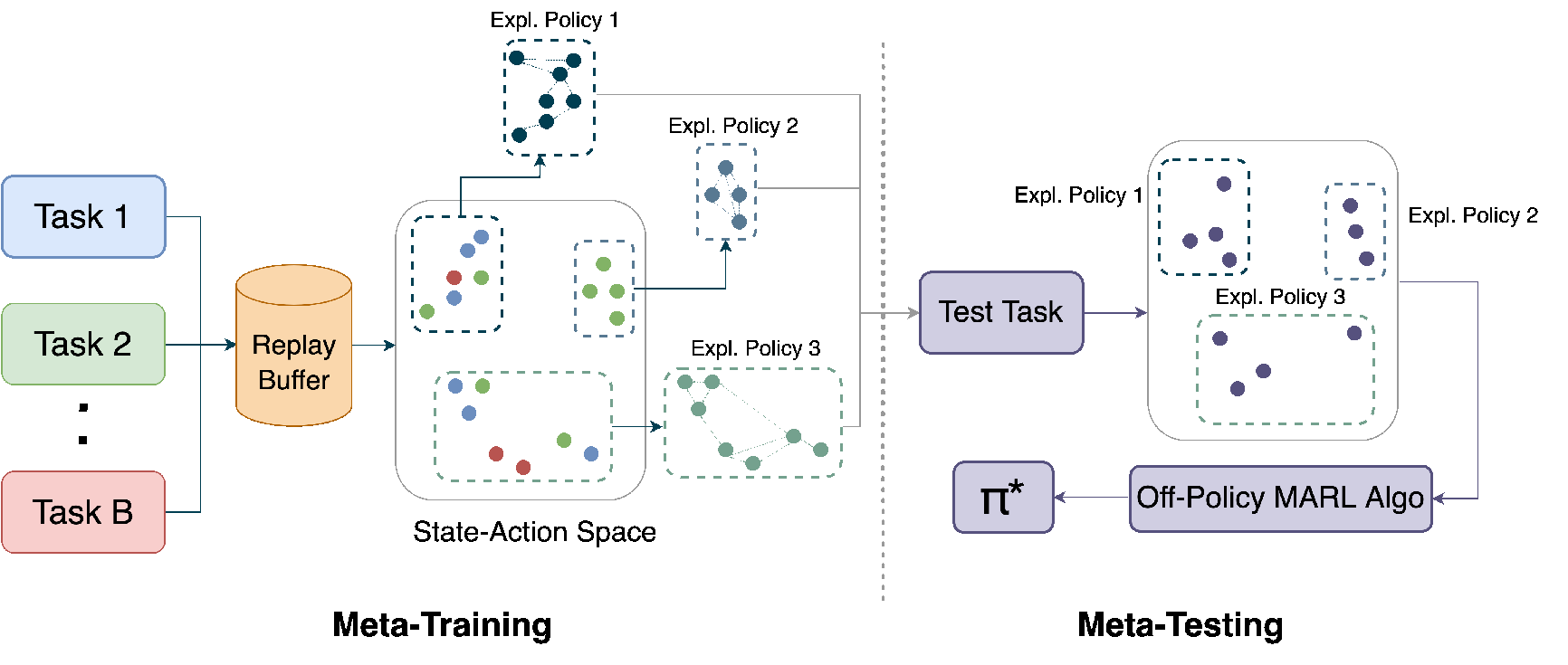}
    \caption{MESA's meta-learning framework. In the meta-training stage, MESA learns exploration policies to cover the high-rewarding subspace. In the meta-testing stage, MESA uses the learned exploration policies to assist the learning in an unseen task. Each color corresponds to a different task, and the colored points represent the high-rewarding joint state-action pairs collected in that task.
    }
    \Description{MESA's meta-learning framework. In the meta-training stage, MESA learns exploration policies to cover the high-rewarding subspace. In the meta-testing stage, MESA uses the learned exploration policies to assist the learning in an unseen task. Each color corresponds to a different task, and the colored points represent the high-rewarding joint state-action pairs collected in that task.}
    \label{fig:pipeline}
\end{figure*}

We analyze a fully cooperative matrix game known as Climb Game. In Section~\ref{section:thereom-challenge}, we show how popular exploration strategies, including unstructured strategies like uniform exploration and task-specific strategies like $\epsilon-$greedy, fail to efficiently explore the climb game. By contrast, we show in Section~\ref{section:thereom-structure} that a simple structured exploration strategy can substantially improve the exploration efficiency.

A climb game $G_f(n, u, U)$ is a $n$-player game with action space %
$\mathcal{A}_i =\{0,\dots,U-1\}$ for any player $i$. The reward of a joint action $\boldsymbol{a}\in \mathcal A$ is determined by the number of players performing a specific action $u$ (denoted as \#$u$), which is
\begin{align}
    R(\boldsymbol{a}) = 
    \begin{cases}
        1,&\text{if \#}u = n,\\
        1 - \delta \text{ } (0 < \delta < 1), &\text{if \#}u = 0,\\
        0,& \text{otherwise}.
    \end{cases}.
\end{align}

\subsection{Exploration Challenge}
\label{section:thereom-challenge}

A climb game $G_f(n,u,U)$ has three groups of NE: the Pareto optimal NE $(u,u,\dots,u)$, the sub-optimal NEs $\{(a_1,a_2,\dots,a_n)\mid \forall i, \ a_i \neq u\}$, and the zero-reward NEs $\{(a_1,a_2,\dots,a_n) \mid 1 < \# u < n\}$. The sheer difference in the size of the three subsets of NEs makes it particularly challenging for RL agents to learn the optimal policy profile without sufficient exploration, as evidenced by the theoretical analysis below and empirical evaluation in Section~\ref{experiments}.

\input{content/analysis}

\subsection{Structured Exploration}
\label{section:thereom-structure}
We will show that it is possible to design a better exploration strategy with some prior knowledge of the climb game structure. Consider a specific structured exploration strategy $p_e^{(t)}(i,j)=U^{-1} \left[ \mathds{1}_{i=j} \right]$, where both agents always choose the same action. With such a strategy, we can quickly find the optimal solution to the game. More formally, we have the following theorem.
\begin{theorem}[structured exploration]
\label{theorem:structured}
In the climb game $G_f(2,0,U)$, under structured exploration $p_e^{(t)}(i,j)=U^{-1} \left[ \mathds{1}_{i=j} \right]$, $q_{\mathcal J^{(T)}}(\mathbf{W}, \mathbf{b}, \mathbf{c}, d)$ is equivalently optimal at step $T=O(1)$.
\end{theorem}

Theorem~\ref{theorem:structured} shows the efficiency of exploration can be greatly improved if the exploration strategy captures a proper structure of the problem, i.e., all agents taking the same action. %
We further remark that by considering a set of similar climb games $\mathcal{G}$, where $\mathcal{G}=\{G_f(2,u,U)\}_{u=0}^{U-1}$, the structured exploration strategy $p_e^{(t)}(i,j)=U^{-1} \left[ \mathds{1}_{i=j} \right]$ can be interpreted as a uniform distribution over the optimal policies of this game set $\mathcal{G}$. 
This interesting fact suggests that we can first collect a set of similarly structured games and then derive effective exploration strategies from these similar games. Once a set of structured exploration strategies are collected, we can further adopt them for fast learning in a novel game with a similar problem structure. 
We take the inspiration here and develop a general meta-exploration algorithm in the next section.

%% file: content/analysis.tex
Consider a 2-agent climb game $G_f(2,0,U)$. A joint action $\boldsymbol{a}$ can be represented by a pair of one-hot vectors $[\mathbf{e}_i, \mathbf{e}_j] \in \{0,1\}^{2U}$.
Let $q(\mathbf{x}, \mathbf{y}; \theta)$ be a joint Q function parameterized by $\theta$ that takes input $\mathbf{x}, \mathbf{y}\in \{0,1\}^U$ and is learned to approximate the reward of the game. We hope the joint Q function has the same optimal policy profile. %

\begin{definition}
We call a joint $Q$ function $q(\mathbf{x}, \mathbf{y}; \theta)$ \emph{equivalently optimal} when $q(\mathbf{e}_0, \mathbf{e}_0;\theta) = \max_{0\le i,j< U} q(\mathbf{e}_i, \mathbf{e}_j;\theta)$. When a joint $Q$ function is equivalently optimal, {one can use it to find the optimal policy.}%
\end{definition}
Since neural networks are difficult to analyze in general \cite{blum1988training}, we parameterize the joint $Q$ function in a quadratic form:
\begin{equation}
q(\mathbf{x}, \mathbf{y};\mathbf{W}, \mathbf{b}, \mathbf{c}, d) = \mathbf{x}^\top \mathbf{W} \mathbf{y} + \mathbf{b}^\top \mathbf{x} + \mathbf{c}^\top \mathbf{y} + d
\end{equation}
A Gaussian prior $p(\mathbf{W})=\mathcal N(\mathbf{W};0, \sigma^2_w I)$ is introduced under the assumption that
a non-linear $\mathbf{W}$ is harder and slower to learn.
Quadratic functions have been used in RL \cite{gu2016continuous,wang2019quadratic} as a replacement for the commonly-used multi-layer perceptron, and there are also theoretical results \cite{du2018power} analyzing neural networks with quadratic activation. For the climb game, it is easy to verify that the quadratic coefficients make the joint $Q$ function sufficiently expressive to perfectly fit the reward function {by setting $\mathbf{W}$ to be the reward matrix}. Therefore, the learning process of $Q$ is mainly affected by how the exploration policy samples the data.

Consider an exploration policy $p_e^{(t)}$ that selects joint action $\boldsymbol a=(i,j)$ at step $t$ with probability $p_e^{(t)}(i,j)$. 
The efficiency of an exploration policy can be measured by the required number of steps for learning an equivalently optimal $Q$ function using the maximum likelihood estimator over the data sampled from $p_e^{(t)}$.
{The learning objective includes both the prior $p(\mathbf{W})$ and the likelihood of prediction error $p(E_{ij})$, where the prediction error $E_{ij}=q(\mathbf{e}_i, \mathbf{e}_j;\cdot)-R_{ij}$. If the prediction error is assumed to be depicted by a Gaussian distribution $p(E_{ij})=\mathcal N(E_{ij};0,\sigma_e^2)$ for every visited joint action $(i,j)$,
then the learning objective for the $Q$ function can be formulated as:}
\begin{align} & \mathcal J^{(T)}(\mathbf{W}, \mathbf{b}, \mathbf{c}, d) \notag \\
=& \mathbb E_{\{(i^{(t)},j^{(t)})\sim p_e^{(t)}\}_{t=1}^T} \log \left( p(\textbf{W})\prod_{t'=1}^{T}p(E_{i^{(t)}j^{(t)}}) \right) \notag \\
=& \sum_{t=1}^T \mathbb E_{(i,j)\sim p_e^{(t)}} \left[ \log \mathcal N(q(\mathbf{e}_i, \mathbf{e}_j;\mathbf{W}, \mathbf{b}, \mathbf{c}, d)-R_{ij}; 0,\sigma^2_e) \right]  \notag \\
+& \log \mathcal N(\mathbf{W};0, \sigma^2_w I) + \text{Const.}
\end{align}

We use $q_{\mathcal J^{(T)}}(\mathbf{W}, \mathbf{b}, \mathbf{c}, d)$ to denote the learned joint $Q$ function that maximizes $\mathcal J^{(T)}$ at step $T$. $q_{\mathcal J^{(T)}}(\mathbf{W}, \mathbf{b}, \mathbf{c}, d)$ is determined by the exploration policy $p_e^{(t)}$ and the exploration steps $T$. Then we have the following theorem for the uniform exploration strategy.

\begin{theorem}[uniform exploration]
\label{theorem:uniform}
Assume $\delta\le \frac{1}{6}, U \ge 3$. Using a uniform exploration policy in the climb game $G_f(2,0,U)$, it can be proved that $q_{\mathcal J^{(T)}}(\mathbf{W}, \mathbf{b}, \mathbf{c}, d)$ will become equivalently optimal only after $T=\Omega(|\mathcal A|\delta^{-1})$ steps.
When $\delta=1$, $T=O(1)$ steps suffice to learn the equivalently optimal joint Q function, suggesting the inefficiency of uniform exploration is due to a large set of sub-optimal NEs.
\end{theorem}

The intuition behind Theorem~\ref{theorem:uniform} is that the hardness of exploration in climb games largely comes from the sparsity of solutions: a set of sub-optimal NEs exist but there is only a single Pareto optimal NE.
Learning the joint $Q$ function can be influenced by the sub-optimal NEs. And if the exploration attempts are not well coordinated, a lot of zero reward would be encountered, making it hard to find the Pareto optimal NE. We also remark that uniform exploration can be particularly inefficient since the term $|\mathcal A|$ can be exponentially large in a multi-agent system. {This indicates that more efficient exploration can potentially be achieved by reducing the search space and identifying a smaller ``critical'' subspace.}

{To formally prove Theorem~\ref{theorem:uniform}, we define $f_1, f_2, f_3$ as the step-averaged probability of taking the joint action in optimal NE, suboptimal NE and zero-reward, respectively. We show that to make the joint $Q$ function equivalently optimal, there is a necessary condition that $f_1, f_2, f_3$ should follow. When $T$ is not large enough, this condition cannot be satisfied. Detailed proof is in Appendix A.2.}
\

Next, we consider the case of another popular exploration paradigm, $\epsilon$-greedy exploration.

\begin{theorem}[$\epsilon$-greedy exploration]
\label{theorem:epsilon}
Assume $\delta\le \frac{1}{32}, U \ge 4, U\ge \sigma_w\sigma_e^{-1}$. In the climb game $G_f(2,0,U)$, under $\epsilon$-greedy exploration with fixed $\epsilon \le \frac12$, $q_{\mathcal J^{(T)}}(\mathbf{W}, \mathbf{b}, \mathbf{c}, d)$ will become equivalently optimal only after $T=\Omega(|\mathcal A|\delta^{-1}\epsilon^{-1})$ steps. If $\epsilon(t)=1/t$, it requires $T=\exp \left( \Omega\left(|\mathcal A| \delta^{-1} \right) \right)$ exploration steps to be equivalently optimal.
\end{theorem}

{The proof is similar to that of Theorem~\ref{theorem:uniform} (detailed in Appendix A.3).} By comparing \ref{theorem:uniform} and \ref{theorem:epsilon}, $\epsilon$-greedy results in even poorer exploration efficiency than uniform exploration. 
Note the $\epsilon$-greedy strategy is training policy specific, i.e., the exploration behavior varies as the training policy changes.
Theorem \ref{theorem:epsilon} suggests that when the policy is sub-optimal, the induced $\epsilon$-greedy exploration strategy can be even worse than uniform exploration. 
Hence, it can be beneficial to adopt a separate exploration independent from the training policy. 

{The above analysis shows that common exploration strategies like uniform exploration or $\epsilon$-greedy exploration are inefficient for such a simple game and the main reason is that it requires coordination between different agents to reach high-rewarding states, but naive exploration strategies lack such cooperation.}

%% file: content/methods.tex
\section{Method}

We detail our method \textit{Cooperative Meta-Exploration in Multi-Agent Learning through Exploiting State-Action Space Structure} (MESA) for cooperative multi-agent learning. 
As shown in Figure \ref{fig:pipeline}, MESA consists of a meta-training stage (Algo.~ \ref{alg:meta_training_algorithm}) and a meta-testing stage (Algo.~\ref{alg:meta_testing_algorithm}). In the meta-training stage, MESA learns exploration policies by training in a batch of training tasks that share intrinsic structures in the state-action space. In the meta-testing stage, MESA utilizes the meta-learned exploration policies to assist learning in an unseen task sampled from the distribution of the training tasks.

\begin{figure}[tb]
\input{content/meta_train_algo}
\end{figure}

\begin{figure}[tb]
\input{content/meta_test_algo}
\end{figure}

\subsection{Meta-Training}
The meta-training stage 
contains two steps: 1) identify the high-rewarding state-action subspace, and 2) train a set of exploration policies using the subspace-induced rewards.

\subsubsection{Identifying High-Rewarding Joint State-Action Subspace}

For each training task $\mathcal{T}_i$, we collect experiences $\mathcal{D}_i = \{(s_t,\boldsymbol{a}_t,r_t,s_{t+1})\}$. If the reward $r_t$ is higher than a threshold $R^\star$, we call this joint state-action pair $(s_t,\boldsymbol{a}_t)$ valuable and store it into a dataset $\mathcal{M}_*$. 
For goal-oriented tasks where $r=\mathds{1}_{s=goal}$, the threshold can be set as $R^\star=1$. For other tasks, the threshold can be set as a hyperparameter, for example, a certain percentile of all collected rewards. {A smaller $R^\star$ results in a larger identified subspace but a less efficient exploration policy.}

{The data stored in $\mathcal{M}_*$ is highly diversified since it comes from all the $B$ training tasks, which are expected to share an intrinsic structure.
}
{We expect that with this intrinsic structure, the high-rewarding joint state-action pairs fall into some low-dimensional subspace.}
In the simplest case, they may form several dense clusters, or many of them lie in a hyperplane. Even if the subspace is not easily interpretable to humans, it may still be effectively ``covered'' by a set of exploration policies (to be found in the subsequent step).

We also explicitly deal with the reward sparsity problem by assigning a positive reward to a joint state-action pair $(s_t,\boldsymbol{a}_t)$ if it has zero reward but leads to a valuable state-action pair $(s_{t'},\boldsymbol{a}_{t'})$ later in the same trajectory. We also put these relabeled pairs into the dataset $\mathcal{M}_*$. Let $t' = \arg\min_{t'>t} [{r}_{t'} > 0]$, we therefore have the following densified reward function
\begin{equation} \label{eq:relabel_reward_dense}
    \begin{aligned}
        \hat{r}_t = \begin{cases}
            \gamma^{t'-t} \cdot {r}_{t'}, &{r}_t = 0, \\
            {r}_t, &{r}_t > 0.
        \end{cases}
    \end{aligned}
\end{equation}

\subsubsection{Learning Exploration Policies} 
In this step, we aim to learn a diverse set of exploration policies to cover the identified high-rewarding joint state-action subspace. We use a distance metric $\Vert\cdot\Vert_{\mathcal F}$ (e.g., $l_2$ distance) to determine whether two state-action pairs are close.
Then if a visited joint state-action pair $(s,\boldsymbol{a})$ is close enough to the identified subspace $\mathcal M_*$, i.e., $\min_{d\in \mathcal{M}_{*}} \Vert (s,\boldsymbol{a}),d\Vert_{\mathcal F} < \epsilon$, it would be assigned a derived positive reward $\hat r$. Increasing the value of $B$ in the collection step would generally result in a more accurate distance measurement. However, this comes at the cost of making the minimization calculation more computationally expensive.

To encourage a broader coverage of the subspace and to avoid mode collapse, the reward assignment scheme ensures that repeated visits to similar joint state-action pairs within one trajectory would result in a decreasing reward for each visit. Similar to \cite{tang2017exploration}, we adopt a pseudo-count function $N$ with a hash function $\phi(\boldsymbol{s},\boldsymbol{a})$ to generalize between similar joint state-action pairs. We then apply a decreasing function $f_d: \mathcal{N} \mapsto [0,1]$ on the trajectory-level pseudo-count $N(\phi((s, \boldsymbol{a}))$. The resulted reward assignment scheme is defined as follows:
\begin{equation} \label{eq:relabel_reward}
    \tilde{r}_t = \hat{r}_t f_d(N(\phi((s_t, \boldsymbol{a}_t))) \left[ \mathds{1}_{\min_{d\in \mathcal{\mathcal M_*}} \Vert (s_t,\boldsymbol{a}_t),d\Vert_{\mathcal F} < \epsilon} \right]
\end{equation}

After one exploration policy is trained with this reward, we will train a new policy to cover the part of the identified subspace that has not yet been covered.
This is achieved by having a global pseudo-count $\hat{N}$ which is updated after training each exploration policy using its visitation counts and is maintained throughout the training of all exploration policies.
This iterative process continues until the subspace is well-covered by the set of trained exploration policies.

\subsection{Meta-Testing}

During meta-testing, MESA uses the meta-learned exploration policies $\{\boldsymbol{\pi}_e^i\}_{i=1}^E$ to assist the training of any generic off-policy MARL algorithm on a test-time task $\hat{\mathcal{T}}$. Specifically, for each rollout episode, we choose with probability $\epsilon$ to execute one uniformly sampled exploration policy $\boldsymbol{\pi}_e \sim \mathcal{U}(\{\boldsymbol{\pi}_e^i\}_{i=1}^E)$. 
For the best empirical performance, we also adopt an annealing schedule $\epsilon: T \mapsto [0,1]$ so that the exploration policies provide more rollouts at the initial stage of the training and are gradually turned off later. %

Here we further provide some analysis of deploying the meta-learned exploration policy on unseen testing tasks.

\begin{theorem}[Exploration during Meta-Testing]
 \label{theorem:exploration_generalize}
Consider goal-oriented tasks with goal space $\mathcal G \subseteq \mathcal S$. Assume the training and testing goals are sampled from the distribution $p(x)$ on $\mathcal G$, and the dataset has $N$ i.i.d. goals sampled from a distribution $q(x)$ on $\mathcal S$. If the exploration policy generalizes to explore $\epsilon$ nearby goals for every training sample, we have that the testing goal is not explored with probability at most
\begin{equation}
P_{\text{fail}}\approx \int p(x)(1-\epsilon q(x))^N dx \le O\left( \frac{KL(p||q)+\mathcal H(p) }{\log(\epsilon N)} \right).
\end{equation}
\end{theorem}

Theorem~\ref{theorem:exploration_generalize} shows that the good performance of meta-learned exploration policy relies on 1) a small difference between the training and testing distribution; and 2) a structured, e.g., low-dimensional, high-rewarding subspace $\mathcal G$ to reduce $\mathcal H(p)$. And when uniformly sampling the training data, $KL(p||q)$ is bounded by $\log \Omega_\mathcal{G}$ in our method. This term, however, can be up to $\log \Omega_\mathcal{S}$ with an uncoordinated exploration on the joint state space $\mathcal{S}$, where $\Omega_\mathcal{S}$ can be exponentially larger than $\Omega_\mathcal{G}$.

\subsection{Implementation Detail of MESA}
{We choose MADDPG, following the centralized training with decentralized execution (CTDE) paradigm, as the off-policy MARL algorithm for MESA since it can be applied to both discrete and continuous action space, as shown in its original paper~\cite{MADDPG}.}
We use a clustering mapping $f_c$ as the hash function $\phi$ so that the dataset $\mathcal{M}_*$ is clustered into $C$ clusters defined by the clustering function $f_c: \mathcal{S}\times\mathcal{A} \mapsto [C]$. The cluster mapping is implemented with the KMeans clustering algorithm~\cite{lloyd1982least}. The number of exploration policies to learn is viewed as a hyperparameter. See the Appendix for detailed hyperparameter settings.

%% file: content/meta_train_algo.tex
\begin{algorithm}[H]
\caption{MESA: Meta-Training}
\label{alg:meta_training_algorithm}
\textbf{Input}: Meta-training tasks $\{\mathcal{T}_i\}_{i=1}^B \sim p(\mathcal{T})$, off-policy MARL algorithm $f$, distance metric $\Vert \cdot \Vert_{\mathcal{F}}$ \\
\textbf{Parameter}: \#policies $E$, threshold $R^\star$, horizon $h$ \\
\textbf{Output}: Exploration policies $\{\boldsymbol{\pi}_e^i\}_{i=1}^E$

\begin{algorithmic}[1] %
\STATE $\mathcal{M}_* \gets \emptyset$, global pseudo-count $\hat{N} \gets 0$
\FOR{i = 1 to B}
    \STATE Initialize policy $\boldsymbol{\pi}_\theta$
    \STATE Train $\boldsymbol{\pi}_\theta$ with $f$ and collect dataset $D_i = \{(\boldsymbol{s}_t,\boldsymbol{a}_t,r_t,\boldsymbol{s}_{t+1})\}$
    \STATE $\mathcal{M}_* \gets \mathcal{M}_* \cup \{\tau \mid R(\tau) \geq R^\star, \tau \in D_i\}$
\ENDFOR
\FOR{i = 1 to E}
    \STATE Initialize exploration policy $\boldsymbol{\pi}_e^i$
    \WHILE{$\boldsymbol{\pi}_e^i$'s training not converged}
        \STATE Initialize $N$ as $\hat{N}, \mathcal{D}\gets\emptyset$
        \FOR{t = 0 to h-1}
            \STATE Execute $\boldsymbol{a}_t \sim \boldsymbol{\pi}_e^i(s_t)$, and observe $(s_t,\boldsymbol{a_t}$ $,r_t,s_{t+1})$ 
            \STATE Calculate $\hat{r}_t$ based on Eq. \ref{eq:relabel_reward_dense} or \ref{eq:relabel_reward}
            \STATE Store $(s_t,\boldsymbol{a}_t,\hat r_t,s_{t+1})$ into $\mathcal{D}$
            \STATE $N(\phi(s_t,\boldsymbol{a}_t)) \gets N(\phi(s_t,\boldsymbol{a}_t)) + 1$
        \ENDFOR
        \STATE Optimize policy $\boldsymbol{\pi}_e^i$ with algorithm $f$
   \ENDWHILE
   \STATE Update $\hat{N}$ using $\mathcal{D}$
\ENDFOR
\STATE \textbf{return} $\{\boldsymbol{\pi}_e^i\}_{i=1}^E$
\end{algorithmic}
\end{algorithm}

%% file: content/meta_test_algo.tex
\begin{algorithm}[H]
\caption{MESA: Meta-Testing}
\label{alg:meta_testing_algorithm}
\textbf{Input}: %
Test task $\hat{\mathcal{T}}$, meta-trained exploration policies $\{\boldsymbol{\pi}_e^i\}_{i=1}^E$, off-policy MARL algorithm $f$ \\
\textbf{Parameter}: horizon $h$ \\
\textbf{Output}: Policy $\boldsymbol{\pi}_\theta$ for task $\hat{\mathcal{T}}$ 

\begin{algorithmic}[1] %
\STATE Initialize policy $\boldsymbol{\pi}_\theta$, $\mathcal{D} = \emptyset$, annealing $\epsilon$
\WHILE{not converged}
    \STATE Determine $p_e$ under annealing probability schedule $\epsilon$
    \STATE Choose policy to perform rollouts by
            $$
            \boldsymbol{\pi}_d = 
            \begin{cases}
                \boldsymbol{\pi}_e \sim \mathcal{U}(\{\boldsymbol{\pi}_e^i\}_{i=1}^E), &\text{w.p.}\ p_e \\
                \boldsymbol{\pi}_\theta, &\text{otherwise}.
            \end{cases}
            $$
    \FOR{t = 0 to h-1}
        \STATE Execute $\boldsymbol{a}_t \sim \boldsymbol{\pi}_d(s_t)$.
        \STATE Observe transition $(s_t,\boldsymbol{a}_t,r_t,s_{t+1})$.
        \STATE $\mathcal{D} \gets \mathcal{D} \cup (s_t,\boldsymbol{a}_t,r_t,s_{t+1})$
    \ENDFOR
    \STATE Optimize $\boldsymbol{\pi}_\theta$ with algorithm $f$ on replay buffer $\mathcal{D}$
\ENDWHILE
\STATE \textbf{return} $\boldsymbol{\pi}_\theta$
\end{algorithmic}
\end{algorithm}

%% file: content/experiments.tex
\section{Experiments} \label{experiments}

Our experimental evaluation aims to answer the following questions:
{
(1) Are the meta-learned exploration policies capable of achieving more efficient exploration during meta-testing on newly sampled tasks in matrix climb game variants (Section \ref{results_matrix_climb}) and high-dimensional domains (Section \ref{results_MPE_climb} and \ref{mamujoco-results})?} { (2) Can these meta-learned exploration policies successfully generalize to unseen test-time tasks from a more challenging (e.g., with more agents) test task distribution which is different the training task distribution (Section \ref{results-generalize})?
}

\subsection{Evaluation Setup}
\label{sec:environments}

\textbf{Compared Methods.}
We compare to $3$ multi-agent reinforcement learning algorithms: MADDPG \cite{MADDPG}, MAPPO \cite{MAPPO}, and QMIX \cite{QMIX}, to measure the effectiveness of our exploration policies. We also compare to $3$ multi-agent exploration algorithm: MAVEN \cite{MAVEN}, MAPPO with RND exploration~\cite{RND}, and EMC~\cite{zheng2021episodic}.
{To compare with baselines that adopt a similar meta-training stage, we add two naive meta-learning baselines, including one with an unconditioned shared policy, which is trained over all training tasks, and one with a goal-conditioned policy, which takes the target landmarks as parts of the input. We also adapt the single-agent meta-RL algorithm MAESN~\cite{MAESN} to the multi-agent setting. Finally, we adapt the single-agent C-BET \cite{parisi2021interestingdeepak} to multi-agent settings based on MAPPO.}
The training and testing tasks are as defined in Section \ref{sec:environments}. Please refer to the Appendix for more visualization and experimental results.

\textbf{Environments.}
We experiment on the Climb Game, Multi-agent Particle Environment (MPE)~\cite{MADDPG}, and multi-agent MuJoCo~\cite{peng2021facmac}, on which generating a distribution of meta-training tasks $p(\mathcal{T})$ is feasible. %

\begin{figure}[tb]
    \centering
    \adjustbox{trim=0.5cm 0.5cm 0cm 0cm}{%
    \includegraphics[width=0.55\textwidth]{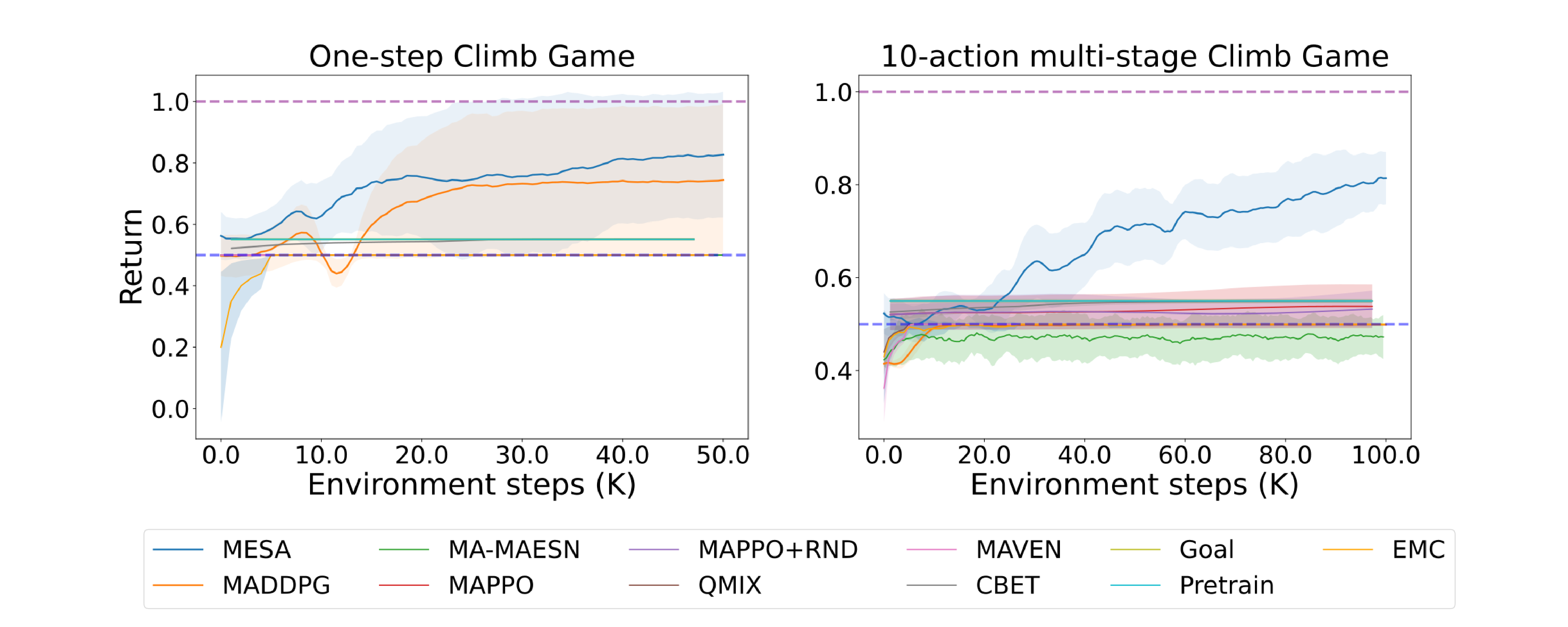}%
    }
    \label{fig:climb_single}
    \caption{Learning curve of the two climb game variants w.r.t number of environment steps. The return is averaged over timesteps for the multi-stage games. The dotted lines indicate the suboptimal return of $0.5$ (purple) and the optimal return $1$ (blue) for each agent. %
    }
    \Description{Learning curve of the two climb game variants w.r.t number of environment steps. The return is averaged over timesteps for the multi-stage games. The dotted lines indicate the suboptimal return of $0.5$ (purple) and the optimal return $1$ (blue) for each agent.}
    \label{fig:climb_game}
\end{figure}

\subsection{Climb Game Variants}

\begin{figure*}[!ht]
    \centering
    \adjustbox{trim=2.3cm 0.5cm 0.5cm 0cm}{%
        \includegraphics[width=1.24\textwidth]{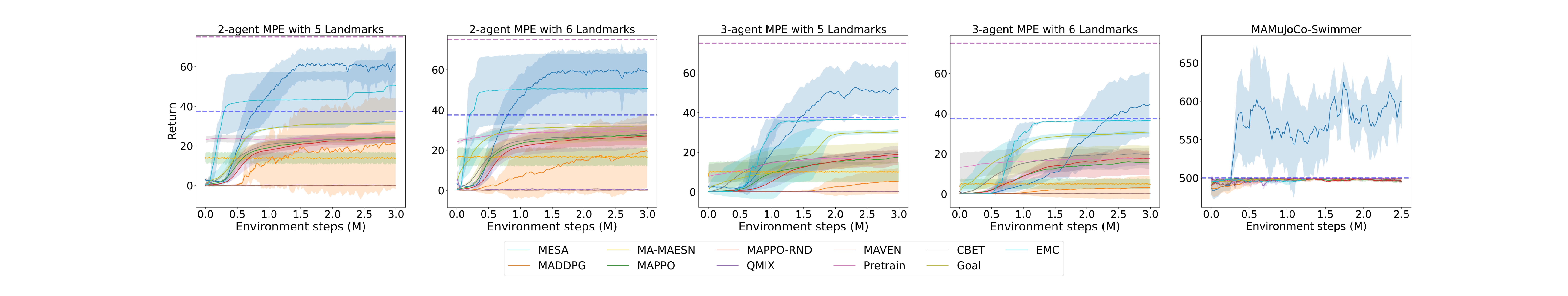}
    }
    \caption{Learning curves of MESA and the compared baselines w.r.t the number of environment interactions during the meta-testing stage in the MPE domain and the multi-agent MuJoCo environment Swimmer. {The two dotted lines indicate the ideal optimal (purple) and sub-optimal (blue) return summed over timesteps. A return above the blue line would typically indicate that the agents are able to learn the optimal strategy.}}
    \Description{Learning curves of MESA and the compared baselines w.r.t the number of environment interactions during the meta-testing stage in the MPE domain and the multi-agent MuJoCo environment Swimmer. {The two dotted lines indicate the ideal optimal (purple) and sub-optimal (blue) return summed over timesteps. A return above the blue line would typically indicate that the agents are able to learn the optimal strategy.}}
    \label{fig:MPE_Mujoco}
\end{figure*}

First, we consider task spaces consisting of variants of the aforementioned climb games.
We extend previous climb game to (1) \textbf{one-step climb game} $G(n,k,u,U)$, which is a $n$-player game with $U$ actions for each player, and the joint reward is $1$ if \#$u = k$, $1 - \delta$ if \#$u = 0$, and $0$ otherwise. The task space $\mathcal{T}^{\text{one}}_U$ consists of all one-step climb games that contain two players and $U$ actions; (2) \textbf{multi-stage climb game}, which is an $S$-stage game where each stage is a one-stage climb game with the same number of available actions. Each stage $t$ has its own configuration $(k_t, u_t)$ of the one-stage climb game $G(2, k_t, u_t, U)$. Agents observe the history of joint actions and the current stage $t$. The task space $\mathcal{T}^{\text{multi}}_{S,U}$ consists of all multi-stage climb games with $S$ stages and $U$ actions.
In our experiments, we use $\mathcal{T}^{\text{one}}_{10}$ and $\mathcal{T}^{\text{multi}}_{5,10}$ as the task space for the one-step and multi-stage Climb Games. We choose uniformly at random ten training tasks and three different test tasks from the task space $\mathcal{T}$, and we keep $\delta=\frac12$ as in the classic climb games.

\textbf{Results on Climb Game Variants.} \label{results_matrix_climb} For the matrix games, we additionally compare with MA-MAESN, which is our adaptation of the original single-agent meta-learning algorithm MAESN~\cite{MAESN} to the multi-agent scenario %
In the single-step matrix game, MESA exhibits better performance, being able to find the optimal reward in some harder tasks when $k = 2$, while other baselines are stuck at the sub-optimal reward for almost all tasks.

In the more challenging $10$-action multi-stage game where task space is exponentially larger, MESA outperforms all compared algorithms by a large margin.
With the help of the exploration policies that have learned the high-rewarding joint action pairs, MESA quickly learns the optimal joint action for each stage and avoids being stuck at the sub-optimal.

\begin{figure}[!h]
    \centering
    \adjustbox{trim=0.5cm 0cm 0.5cm 0cm}{%
        \includegraphics[width=.3\textwidth]{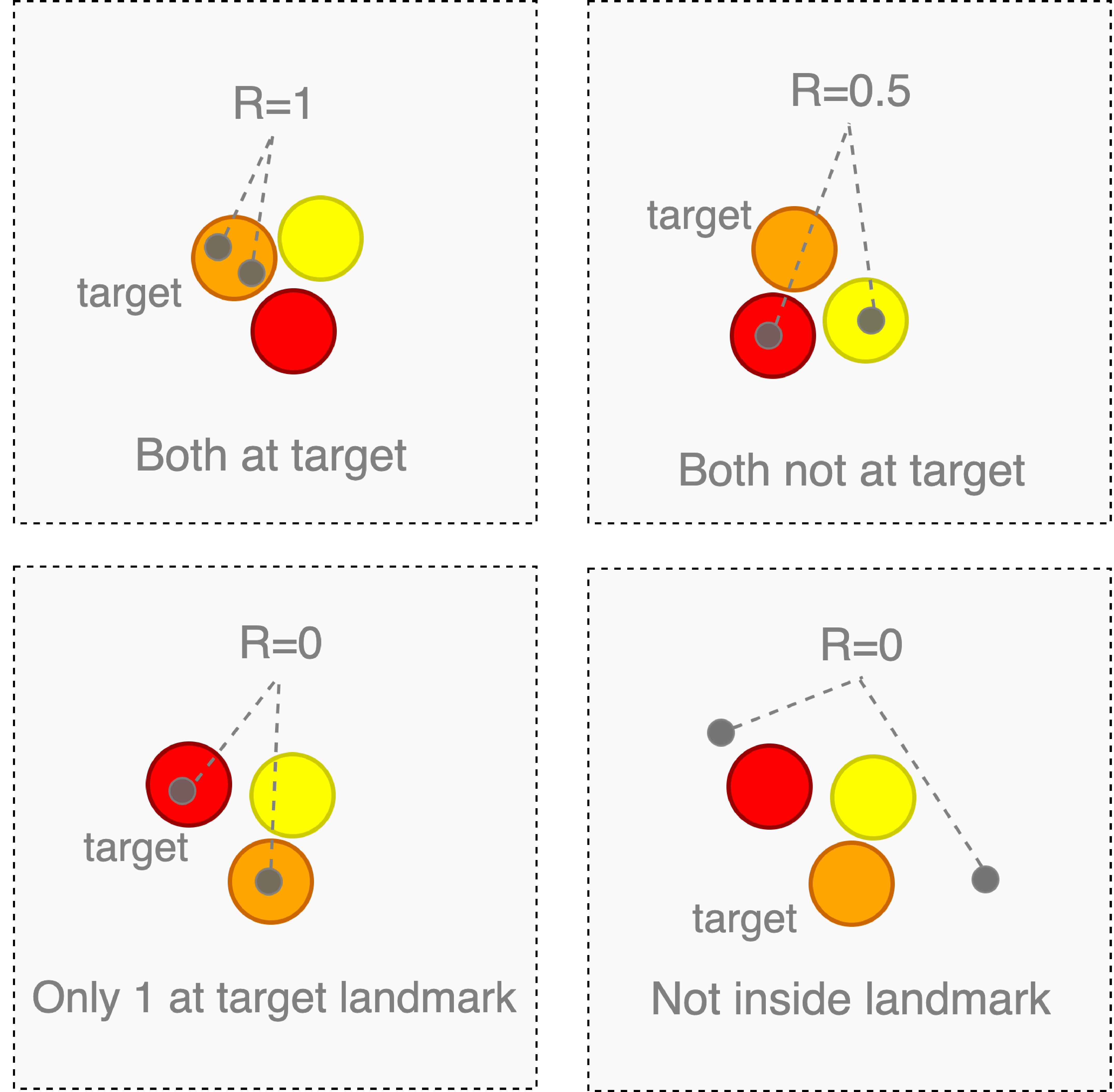}
    }
    
    \caption{Visualizations of a $2$-player $3$-landmark MPE climb game. %
    }
    \Description{Visualizations of a $2$-player $3$-landmark MPE climb game with four types of reward.}
    \label{fig:MPE_Illu}
\end{figure}

\subsection{MPE Domain}

We extend the matrix climb games to MPE \cite{MADDPG}, which has a continuous high-dimensional state space. Agents must first learn to reach the landmarks under sparse rewards and then learn to play the climb games optimally.

In a MPE Climb Game $\bar{G}(n,k,u,U,\{L_j\}_0^{U-1})$ (Figure \ref{fig:MPE_Illu}), there are $U$ non-overlapping landmarks with positions $\{L_j\}_{j=0}^{U-1}$.
The reward is non-zero only when every agent is on some landmark. Agents will be given a reward of $1$ if there are exactly $k$ agents located on the $u$-th landmark (target landmark), and a suboptimal reward of $1-\delta$ will be given when none of the agents are located on the target landmark. Otherwise, the reward will be zero. As before, $u$ and $k$ are not present in the observation and can only be inferred from the received reward.
A task space $\mathcal{T}^{\text{MPE}}_{n,U}$ consists of all MPE climb games with $n$ players and $U$ landmarks.
We evaluate MESA on the $2$-agent tasks ($\mathcal{T}^{\text{MPE}}_{2,5}$ and $\mathcal{T}^{\text{MPE}}_{2,6}$) and $3$-agent tasks ($\mathcal{T}^{\text{MPE}}_{3,5}$ and $\mathcal{T}^{\text{MPE}}_{3,6}$) while fixing $k=2$. Each sampled training and testing task has a different configuration of landmark positions.

\textbf{Adaptation Performance in MPE.} \label{results_MPE_climb} We show in Figure \ref{fig:MPE_Mujoco} the learning curve of our approach MESA compared with the aforementioned baseline methods. MESA outperforms the compared baselines by a large margin, being able to coordinately reach the task landmark quickly, as evidenced by the near-optimal reward. Even when combined with RND-based exploration, MAPPO easily sticks to the sub-optimal equilibrium. Value-based methods like QMIX and MAVEN are unable to learn the correct $Q$-function because the reward is quite sparse before agents can consistently move themselves to a landmark. EMC sometimes jumps out of the suboptimal equilibrium with curiosity-driven exploration, but the performance is not robust. Furthermore, as the meta-learning baselines only learn the sub-optimal behavior during meta-training, they fail to learn the optimal equilibrium during test time and quickly converge to the suboptimal equilibrium.

\textbf{Visualization of Exploration Policies.} 
\label{visualization}
To answer question (2), we visualize the learned exploration policies in a $2$-agent $3$-landmark MPE task in Figure \ref{fig:mode}. We can see that the learned exploration policy consecutively visited the $3$ landmarks within $20$ timesteps in one trajectory.

\begin{figure}[t]
    \centering
    \includegraphics[width=0.35\textwidth]{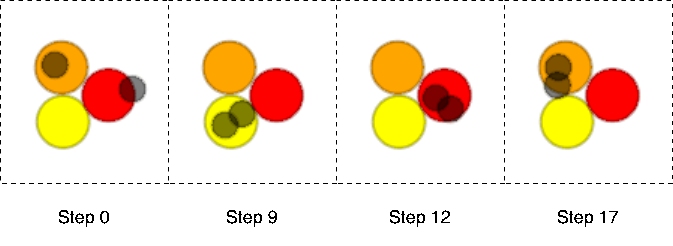}
    \caption{Visualization of structured exploration behaviors discovered by the meta-trained exploration policy in MESA.} %
    \Description{Visualization of structured exploration behaviors discovered by the meta-trained exploration policy in MESA.}
    \label{fig:mode}
\end{figure}

\subsection{Multi-agent MuJoCo Environments}
\label{mamujoco-results}
We also extend the matrix climb games to multi-agent MuJoCo environments~\cite{peng2021facmac}. We consider specifically the $2$-agent Swimmer environment where each agent is a hinge on the swimmer's body, and each agent's action is the amount of torque applied to hinge rotors. The extension considers the angles between the two hinges and the body segments. Each task in the task space is a target angle such that a reward of $1$ will be given only if the two angles are both close to the target angles, a $0.5$ suboptimal reward is given if none of two angles are close to the target, and a reward of $0$ if only one of the two angles are close.

This multi-agent environment is extremely hard as agents are very likely to converge to the suboptimal reward of $0.5$, which is confirmed by the results that none of the baselines were able to
{find the optimal equilibrium}
in Figure \ref{fig:MPE_Mujoco}. Therefore, MESA vastly outperforms all the compared baselines by learning a final policy that frequently reaches the target angle.

\subsection{Generalization Performance of MESA}
\label{results-generalize}

{In this section, our goal is to evaluate the generalization performance of the meta-trained exploration policy in scenarios where the meta-training and meta-testing task distributions are different. In particular, we focus on the setting where the test-time tasks are \emph{more challenging} than the training-time tasks and examine how an exploration policy learned from simpler tasks can boost training performances on harder tasks.}

{
The test task here is uniform on the $3$-agent high-difficulty MPE Climb games. The task difficulty is defined by the average pairwise distances between the landmark positions and the initial positions of the agents. 
We consider two simpler training task distributions, including
(1) a $2$-agent setting with the same difficulty, and (2) a $3$-agent setting with a lower difficulty. In both settings, the meta-training tasks are less challenging than the test-time tasks. For evaluation, the meta-trained exploration policy from each setting will be directly applied to assist the training on the more challenging test-time tasks, \emph{without any fine-tuning}.}

{
We modified the neural network architecture by adopting an attention layer in both actor and critic to ensure they are compatible with a varying number of agents.
The attention mechanism acts as an aggregation function between the relative positions of the other agents and its own relative position to the landmarks to handle the varying observation dimensions.
Additionally, we employed behavior cloning (BC)~\cite{pomerleau1991efficient} on the rollouts of the exploration policies as a warm-up to accelerate learning of the final policy.}

\begin{figure}[h]
    \centering
    \includegraphics[scale=0.16,trim={0 0 0 0}]{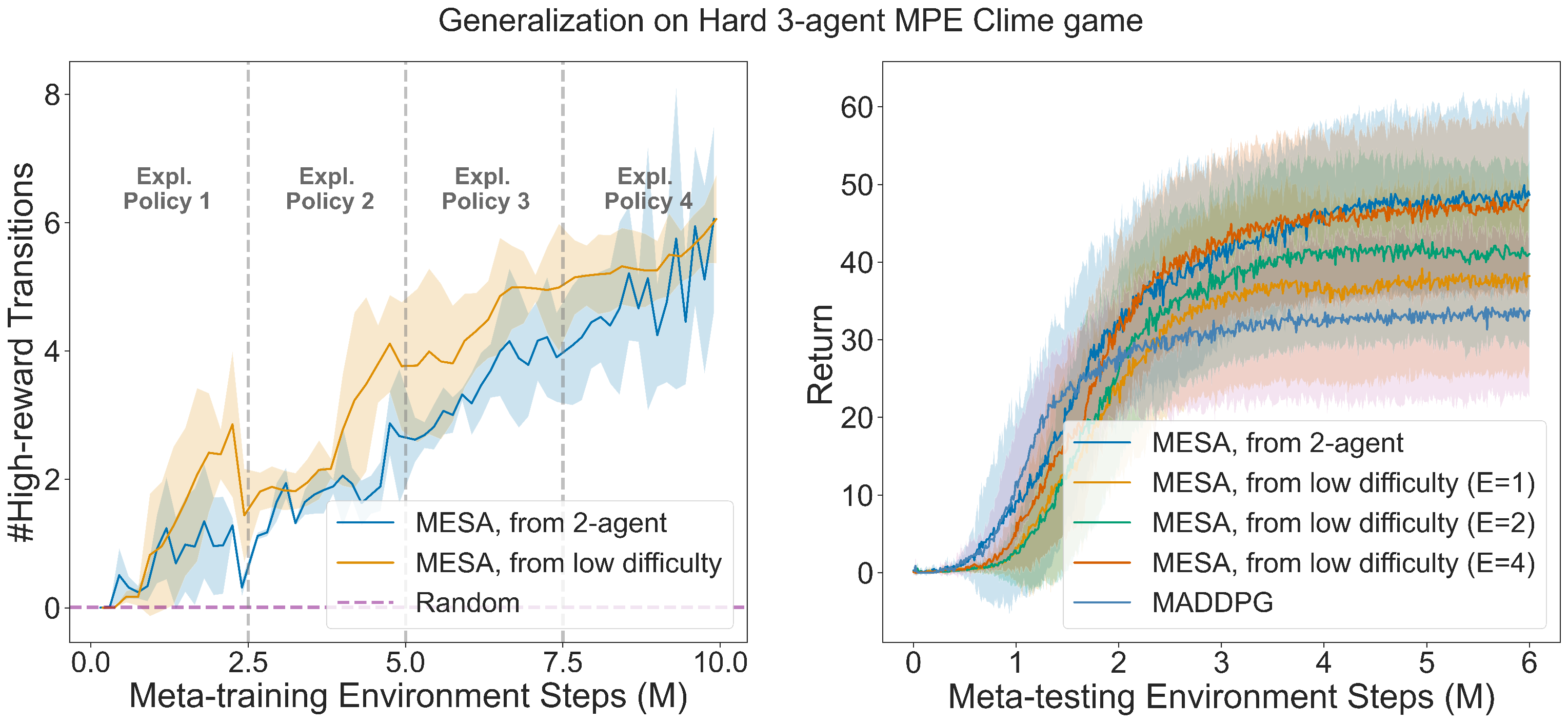}
    \caption{Generalization results of MESA on the hard $3$-agent MPE Climb game. Left: Zero-shot generalizability of the meta-exploration policies, measured by {the number of visitations on high-reward transitions per episode} on the test tasks. The purple dotted line corresponds to the random exploration policy. The plot shows the concatenated training curves for all exploration policies. Right: Learning curves of MESA under different settings using the meta-exploration policies trained on the two different training-task distributions.}
    \Description{Generalization results of MESA on the hard $3$-agent MPE Climb game. Left: Zero-shot generalizability of the meta-exploration policies, measured by {the number of visitations on high-reward transitions per episode} on the test tasks. The purple dotted line corresponds to the random exploration policy. The plot shows the concatenated training curves for all exploration policies. Right: Learning curves of MESA under different settings using the meta-exploration policies trained on the two different training-task distributions.}
    \label{fig:gen}
\end{figure}

{
In Figure \ref{fig:gen}, we present the generalization results from our study. We evaluate the zero-shot generalization ability of the meta-exploration policy by measuring the average number of high-reward transitions hit in a test task randomly sampled from the test task distribution.
As shown on the left of Figure \ref{fig:gen}, the meta-exploration policies are able to explore the test-time tasks much more efficiently than a random exploration policy, even on test-time tasks that are drawn from a harder task distribution. Notably, the generalization ability increases with the number of exploration policies ($B$).
Using the meta-exploration policies trained on the simpler tasks, MESA is able to consistently reach the high-reward region in the unseen hard $3$-agent tasks, as opposed to the vanilla MADDPG algorithm that only learns the sub-optimal equilibrium. We also see that with an increasing number of meta-exploration policies, the performance of MESA increases, but the improvement becomes marginal, while the meta-training time increases linearly with E.}

%% file: content/conclusions.tex
\section{Conclusions}
This paper introduces a meta-exploration method, MESA, for multi-agent learning. The key idea is to learn a diverse set of exploration policies to cover the high-rewarding state-action subspace and achieve efficient exploration in an unseen task. MESA can work with any off-policy MARL algorithm, and empirical results confirm the effectiveness of MESA in climb games, MPE environments, and multi-agent MuJoCo environments {and its generalizability to more complex test-time tasks}. %

%% file: content/appendix.tex
\input{content/append_proofs}

\input{content/implementation}

\input{content/ablation_datasetINIT}

%% file: content/append_proofs.tex
\newpage
\begin{center}
\huge \textbf{Appendix}
    
\end{center}
\section{Proofs}
\label{appendx:proofs}

\subsection{Proof of Lemma on Equivalent Optimality}
\label{appendx:proof-lemma}

\begin{lemma}
In the 2-agent Climb Game with single-agent action space $|\mathcal A|=U$ and reward matrix
$$R=\begin{pmatrix} r & 0 & \cdots & 0 \\ 0 & r(1-\delta) & \cdots & r(1-\delta) \\ \vdots & \vdots & \ddots & \vdots \\ 0 & r(1-\delta) & \cdots & r(1-\delta) \end{pmatrix},$$
for any exploration policy $p_e$ where $p_e^{(t)}(i,j)$ is the probability of trying action $(i,j)$ at time step $t$, given the objective function
\begin{align}
\label{lemma:objective}
& \mathcal J^{(T)}(\mathbf{W}, \mathbf{b}, \mathbf{c}, d) \notag \\
=& \sum_{t=1}^T \mathbb E_{(i,j)\sim p_e^{(t)}} \left[ \log \mathcal N(q(\mathbf{e}_i, \mathbf{e}_j;\mathbf{W}, \mathbf{b}, \mathbf{c}, d)-R_{ij}; 0,\sigma^2_e) \right] \notag \\  +& \log \mathcal N(W;0, \sigma^2_W I) + \text{Constant}
\end{align}
maximized by parameters $\mathbf{W}^*, \mathbf{b}^*, \mathbf{c}^*, d^*$, the joint Q function \\$q(\mathbf{e}_i, \mathbf{e}_j;\mathbf{W}^*, \mathbf{b}^*, \mathbf{c}^*, d^*)$ is equivalently optimal if the following criterion holds
\begin{equation}
\label{lemma:criterion}
    r\delta \ge \left(\frac{f_2}{f_0}-1\right) \frac{m^2}{f_2\lambda +m^2} \frac{ r(2-\delta) }{ 1 + \frac{2f_2(f_1\lambda+2m)m}{f_1(f_2\lambda+m^2)} + \frac{f_2(f_0\lambda+1)m^2}{f_0(f_2\lambda+m^2)} }.
\end{equation}

Here we use
\begin{align*}
    f_0 =& \frac{1}{T} \sum_{t=1}^T p_e^{(t)}(0,0) \\
    f_1 =& \frac{1}{T} \sum_{t=1}^T\sum_{i=1}^{U-1} \left(p_e^{(t)}(0,i)+p_e^{(t)}(i,0)\right) \\
    f_2 =& \frac{1}{T} \sum_{t=1}^T\sum_{i=1}^{U-1}\sum_{j=1}^{U-1} p_e^{(t)}(i,j) \\
    m =& U - 1 \\
    \lambda =& \frac{T\sigma_{w}^2}{\sigma_e^2}
\end{align*}
for a clearer demonstration of the criterion.

\end{lemma}

\begin{proof}

From the symmetry of the parameters and the concavity of the objective function, $\exists W_0,W_1,W_2,B,C,D$ such that
\begin{align*}
    W_0 =& \mathbf{W}^*_{00} \\
    W_1 =& \mathbf{W}^*_{0i} = \mathbf{W}^*_{i0}, &\forall i\ne 0 \\
    W_2 =& \mathbf{W}^*_{ij}, &\forall i,j \ne 0 \\
    B =& \mathbf{b}^*_0 = \mathbf{c}^*_0 \\
    C =& \mathbf{b}^*_i = \mathbf{c}^*_i, &\forall i\ne 0 \\
    D =& d
\end{align*}

Rewrite the objective function (\ref{lemma:objective}) we obtain

\begin{align}
\label{lemma:principle}
    \mathcal J =& -\frac{T}{2\sigma_e^2} \left( f_0(W_0+2B+D-r)^2 \right.\notag \\
    +& f_1(W_1+B+C+D)^2  \notag \\
    +& \left. f_2(W_2 + 2C + D - r(1-\delta))^2 \right) \notag \\
    -& \frac{1}{2\sigma_w^2}\left( W_0^2 + 2mW_1^2 + m^2W_2^2 \right)
\end{align}

Further, let 
\begin{align*}
    K_0 =& 2B + D - r \\
    K_1 =& B + C + D \\
    K_2 =& 2C + D - r(1-\delta)
\end{align*}
and immediately
\begin{equation}
\label{lemma:basic-eq}
    K_0 + K_2 = 2K_1 - r(2-\delta)
\end{equation}

Following equation (\ref{lemma:principle}),
\begin{align*}
    \left( 2\frac{\partial}{\partial W_0} + \frac{\partial}{\partial W_1} - \frac{\partial}{\partial B} \right)\mathcal J = 0 \Rightarrow & W_0 = - m W_1 \\
    \left( 2\frac{\partial}{\partial W_2} + \frac{\partial}{\partial W_1} - \frac{\partial}{\partial C} \right)\mathcal J = 0 \Rightarrow & W_1 = - m W_2 \\
    \frac{\partial}{\partial W_0} \mathcal J = 0 \Rightarrow & W_0 = -\frac{f_0\lambda}{f_0\lambda + 1}K_0 \\ \frac{\partial}{\partial W_1} \mathcal J = 0 \Rightarrow & W_1 = -\frac{f_1\lambda}{f_1\lambda + 2m}K_1 \\ \frac{\partial}{\partial W_2} \mathcal J = 0 \Rightarrow & W_2 = -\frac{f_2\lambda}{f_2\lambda + m^2}K_2 \\
    \frac{\partial}{\partial B} \mathcal J = 0 \Rightarrow & \frac{W_0 + K_0}{W_1 +K_1} = -\frac{f_1}{2f_0} \\
    \frac{\partial}{\partial C} \mathcal J = 0 \Rightarrow & \frac{W_1 + K_1}{W_2 +K_2} = -\frac{2f_2}{f_1}
\end{align*}
and together with equation (\ref{lemma:basic-eq}), we obtain
\begin{align}
    K_2 = \frac{ r(2-\delta) }{ 1 + \frac{2f_2(f_1\lambda+2m)m}{f_1(f_2\lambda+m^2)} + \frac{f_2(f_0\lambda+1)m^2}{f_0(f_2\lambda+m^2)} }.
\end{align}

Finally we deduce the criterion
\begin{align*}
    & W_0 + 2B + D \ge W_2 + 2C + D \\
    \Leftrightarrow & r\delta \ge \left(1-\frac{f_2}{f_0} \right)(W_2+K_2) \\
    \Leftrightarrow & r\delta \ge \left(\frac{f_2}{f_0}-1\right) \frac{m^2}{f_2\lambda +m^2} \frac{ r(2-\delta) }{ 1 + \frac{2f_2(f_1\lambda+2m)m}{f_1(f_2\lambda+m^2)} + \frac{f_2(f_0\lambda+1)m^2}{f_0(f_2\lambda+m^2)} }.
\end{align*}

\end{proof}

\subsection{Proof for Theorem \ref{theorem:uniform} (uniform exploration)}

\textbf{Theorem \ref{theorem:uniform}.}
\emph{
Assume $\delta\le \frac{1}{6}, U \ge 3$. In the Climb Game $G_f(2,0,U)$, given the quadratic joint Q function form $q(\mathbf{x}, \mathbf{y};\mathbf{W}, \mathbf{b}, \mathbf{c}, d)$ and a Gaussian prior $p(\mathbf{W})=\mathcal N(\mathbf{W};0, \sigma^2_w I)$, using a uniform exploration policy, $q_{\mathcal J^{(T)}}(\mathbf{W}, \mathbf{b}, \mathbf{c}, d)$ will become equivalently optimal only after $T=\Omega(|\mathcal A|\delta^{-1})$ steps.
When $\delta=1$, $T=O(1)$ steps suffice to learn the equivalently optimal joint Q function, meaning the inefficiency of uniform exploration is due to a large set of suboptimal NEs.
}

\label{appendix:proof-uniform}

\begin{proof}

 Under uniform exploration, $$f_0=\frac{1}{U^2}, f_1=\frac{2m}{U^2}, f_2=\frac{m^2}{U^2}.$$ Criterion (\ref{lemma:criterion}) can be reformulated to
\begin{equation}
\label{proof:form1}
    \delta \ge \frac{(m^2-1)(2-\delta)}{ \left(1+\frac{\lambda}{U^2} \right)(m+1)^2}
\end{equation}
and thus with $\lambda = \frac{T\sigma_{w}^2}{\sigma_e^2},m\ge 2,\delta \le \frac{1}{6}$,
\begin{align*}
    T \ge & \frac{U^2 \sigma_e^2}{\sigma_w^2} \left( \frac{(m^2-1)(2-\delta)}{(m+1)^2} - 1 \right) \\
    \ge & \frac{U^2 \sigma_e^2}{\sigma_w^2} \left( \frac{3}{\delta} - \frac{6}{\delta} \right) \\
    = & \frac{U^2 \sigma_e^2}{6\sigma_w^2\delta}
\end{align*}

On the other hand, in non-penalty Climb Game where $\delta=1$, if at any time step $\exists (i,j)\ne(0,0)$ where the joint Q function $q_{\mathcal J^{(T)}}$ weighs action $(a_i,a_j)$ more than the action $(a_0,a_0)$, just swap the parameters related to $(i,j)$ with thos related to $(0,0)$ and the objective function $\mathcal J^{(T)}$ will be increased, which makes a contradiction. Hence, $T=1$ suffices for the non-penalty Climb Game.

\end{proof}

\subsection{Proof for Theorem \ref{theorem:epsilon} ($\epsilon$-greedy exploration)}

\textbf{Theorem \ref{theorem:epsilon}}
\emph{
Assume $\delta\le \frac{1}{32}, U \ge \max(4, \sigma_w\sigma_e^{-1})$. In the Climb Game $G_f(2,0,U)$, given the quadratic joint Q function form \\ $q(\mathbf{x}, \mathbf{y};\mathbf{W}, \mathbf{b}, \mathbf{c}, d)$ and a Gaussian prior $p(\mathbf{W})=\mathcal N(\mathbf{W};0, \sigma^2_w I)$, under $\epsilon$-greedy exploration with fixed $\epsilon \le \frac12$, $q_{\mathcal J^{(T)}}(\mathbf{W}, \mathbf{b}, \mathbf{c}, d)$ will become equivalently optimal only after $T=\Omega(|\mathcal A|\delta^{-1}\epsilon^{-1})$ steps. If $\epsilon(t)=1/t$, it requires $T=\exp \left( \Omega\left(|\mathcal A| \delta^{-1} \right) \right)$ exploration steps to be equivalently optimal.
}
\begin{proof}
\label{appendix:proof-epsilon}

Under the circumstances here, after the first step of uniform exploration, the sub-optimal policy will be used for $\epsilon$-greedy exploration. Then for both fixed $\epsilon$ or linearly decaying $\epsilon$, the following always holds:

\begin{align*}
    \frac{f_1}{f_0} =& 2m \\
    \frac{f_2}{f_0} \ge & \max(2, m^2) \\
    f_2 \ge & \min(1-\epsilon, m^2U^{-2}) \ge \frac12.
\end{align*}

Then it can be derived from criterion \ref{lemma:criterion} that

\begin{align}
\label{proof:form2}
    r\delta \ge& \left(\frac{f_2}{f_0}-1\right) \frac{m^2}{f_2\lambda +m^2} \frac{ r(2-\delta) }{ 1 + \frac{2f_2(f_1\lambda+2m)m}{f_1(f_2\lambda+m^2)} + \frac{f_2(f_0\lambda+1)m^2}{f_0(f_2\lambda+m^2)} } \notag \\
    =& 
    \left(\frac{f_2}{f_0}-1\right) \frac{m^2}{f_2\lambda +m^2} \frac{ r(2-\delta) }{ 1 +  \frac{f_2(f_0\lambda+1)(m^2 + 2m)}{f_0(f_2\lambda+m^2)} } \notag  \\
    =& 
    \left(\frac{f_2}{f_0}-1\right) \frac{m^2}{f_2\lambda +m^2} \frac{ r(2-\delta) }{ (m+1)^2 } \\
    \ge& (m^2 - 1) \frac{m^2}{\lambda +m^2} \frac{ r(2-\delta) }{ (m+1)^2 } \notag
\end{align}

Similar to inequality (\ref{proof:form1}), this yields to

\begin{align}
    \lambda \ge \frac{m^2}{6\delta}
\end{align}

Following inequality (\ref{proof:form2}), we further get
\begin{align*}
    r\delta \ge & 
    \left(\frac{f_2}{f_0}-1\right) \frac{m^2}{f_2\lambda +m^2} \frac{ r(2-\delta) }{ (m+1)^2 } \notag \\
    \ge & \frac{f_2}{2f_0} \frac{m^2}{\lambda + \lambda} \frac{r}{4m^2} \notag \\
    \ge &  \frac{r}{16f_0\lambda},
\end{align*}
which is
\begin{equation}
\label{proof:form3}
\lambda \ge \frac{\delta^{-1}}{16f_0}
\end{equation}

For fixed $\epsilon$, $$f_0 \le \frac{\epsilon}{U^2} + \frac{1}{TU^2} \le \frac{\epsilon}{U^2} + \lambda^{-1},$$ and further 
\begin{align*}
    &\lambda \ge \frac{\delta^{-1}}{16(\frac{\epsilon}{U^2} + \lambda^{-1})} \\
    \Rightarrow& \lambda \ge \frac{U^2}{\epsilon}\left(\frac{\delta^{-1}}{16} - 1\right) \ge \frac{U^2\delta^{-1}}{32\epsilon}
\end{align*}
which shows that $$T=\Theta(\lambda)=\Omega(U^2\delta^{-1}\epsilon^{-1}).$$

When $\epsilon=\frac{1}{T}$, $$f_0 \le \frac{1}{U^2}\frac{\sum_{t=1}^T \frac{1}{t}}{T} \le \frac{2\log(T)}{U^2T}$$ and further
\begin{align*}
    &\lambda \ge \frac{\delta^{-1}}{16\frac{2\log(T)}{U^2T}} \\
    \Rightarrow& 
    \log(T) \ge \frac{U^2\sigma_e^2}{32\sigma_w^2\delta}
\end{align*}
which shows that $$T=\exp(\Omega(U^2\delta^{-1})).$$

\end{proof}

\subsection{Proof for Theorem \ref{theorem:structured} (Structured exploration)}

\textbf{Theorem \ref{theorem:structured}}
\emph{
In the Climb Game $G_f(2,0,U)$, given the quadratic joint Q function form $q(\mathbf{x}, \mathbf{y};\mathbf{W}, \mathbf{b}, \mathbf{c}, d)$ and a Gaussian prior $p(\mathbf{W})=\mathcal N(\mathbf{W};0, \sigma^2_w I)$, under structured exploration $p_e^{(t)}(i,j)=U^{-1} \left[ \mathds{1}_{i=j} \right]$, $q_{\mathcal J^{(T)}}(\mathbf{W}, \mathbf{b}, \mathbf{c}, d)$ is equivalently optimal at step $T=O(1)$.
}

\begin{proof}
\label{appendix:proof-structure}

It is easy to verify that $\mathbf{W}=\mathbf{c}=0, \mathbf{b}=(1,0,....,0)^\top, d=0$ is the learned parameter that maximizes both the prior of $\mathbf{W}$ and the likelihood of prediction error at any step $T$. This parameter configuration directly gives the joint $Q$ function that is equivalently optimal.

\end{proof}

\subsection{Proof for Theorem \ref{theorem:exploration_generalize} (Exploration during Meta-Testing)}

\begin{definition}[$\epsilon$ Generalization]
Suppose there are training and testing data from the same space $\mathcal S.$ Let $g(x,y) \in \{0,1\}$ denote whether a exploration policy trained on training sample $x$ can learn to explore $y$ during testing time. And we always assume $g(x,y)=g(y,x)$. Then we say a exploration policy generalizes to explore $\epsilon$ nearby goals if for every training sample $x$, $\exists$ a neighbourhood $|\Omega(x)| \ge \epsilon$ of $x$ s.t. $\forall y\in \Omega(x)$, $g(x,y)=1$. Intuitively, that means the exploration policies learns to explore $\epsilon$ nearby region of every training sample.
\end{definition}

\textbf{Theorem \ref{theorem:exploration_generalize}}[Exploration during Meta-Testing]
\emph{
Consider goal-oriented tasks with goal space $\mathcal G \subseteq \mathcal S$. Assume the training and testing goals are sampled from the distribution $p(x)$ on $\mathcal G$, the dataset has $N$ i.i.d. goals sampled from a distribution $q(x)$ on $\mathcal S$. If the exploration policy generalizes to explore $\epsilon$ nearby goals for every training sample, we have that the testing goal is not explored with probability at most
\begin{equation}
P_{\text{fail}}\approx \int p(x)(1-\epsilon q(x))^N dx \le O \left( \frac{KL(p||q)+\mathcal H(p) }{\log (\epsilon N)} \right).
\end{equation}
Here we make the assumption that $\epsilon$ is small and $g$ is Lipschitz continuous.
}

\begin{proof}
For any testing goal $x$, every training sample $t\in \Omega(x)$ enables the exploration policy to explore $y$ during testing time. As $\Omega(x)$ is a neighborhood of $x$ and $g$ is Lipschitz continuous, we can select $\epsilon$ training samples $t$ from $\Omega(x)$ that is closest to $x$, and we think those samples $t$ has a similar sampling density function, i.e., $g(t)\approx g(x)$. Thus, with $N$ i.i.d samples, there is approximately $(1-\epsilon q(x))^N$ probability that testing goals $x$ will not be explored during the testing time.
Then we have
\begin{align*}
    P_{\text{fail}} \approx& \int p(x)(1-\epsilon q(x))^N dx \\
    \le& \int p(x)e^{-N\epsilon q(x)} dx \\
    \le& \int p(x) \frac{\log \frac{1}{q(x)} +2 q(x) + 16}{\frac12 \log (\epsilon N) } dx \\
    =& O \left( \frac{KL(p||q)+\mathcal H(p) }{\log (\epsilon N)} \right)
\end{align*}
\end{proof}

You may refer to the following lemma which is used in the above proof.

\begin{lemma}

$\forall k>16, x > 0$,
\begin{equation}
\label{lemma:main_inequal}
e^{-kx} \le \frac{ \log \frac{1}{x} }{\frac12 \log k} + \frac{x}{\frac12 \log k} + \frac{x}{k}
\end{equation}

\end{lemma}

\begin{proof}

Let $x_0 = \frac{2\log k}{k}$. We can prove inequality~\ref{lemma:main_inequal} by proving the following three conditions.

\begin{equation}\label{lemma_inequal:condition1}
1.~~\forall x \ge x_0, \frac{x}{k} \ge e^{-kx}
\end{equation}

\begin{equation}\label{lemma_inequal:condition2}
2.~~\forall 0\le x\le x_0, \frac{ \log \frac{1}{x} }{\frac12 \log k} \ge e^{-kx}
\end{equation}

\begin{equation}\label{lemma_inequal:condition3}
3.~~\forall x>0, x + \log \frac{1}{x} \ge 0.
\end{equation}

To prove (\ref{lemma_inequal:condition1}), it suffices to show $$\frac{x_0}{k} = \frac{2\log k}{k^2} \ge \frac{1}{k^2} = e^{-kx_0},$$ as $\frac{x}{k}$ is monotone increasing and $e^{-kx}$ is monotone decreasing.

Now we prove (\ref{lemma_inequal:condition2}). Let
$$
f(x) = \frac{\log \frac{1}{x}}{\frac12 \log k} - e^{-kx}.
$$

Since
$$
xf'(x) = kxe^{-kx} - \frac{2}{\log k}
$$
is increasing for $x\in (0,1/k)$ and decreasing for $x\in (1/k, +\infty)$, $\exists 0<x_1<1/k<x_2$ s.t. $f'(x)>0 \Leftrightarrow x_1 \le x \le x_2$. Here $x_1,x_2$ are two solutions of $kxe^{-kx} - 2 / \log k = 0$.

Therefore, to prove (\ref{lemma_inequal:condition2}), it suffices to show $f(x_1)\ge 0$ and $f(x_0) \ge 0$, and the later one can be verified as
\begin{align*}
f(x_0) =& \frac{\log \frac{k}{2\log k}}{\frac12 \log k} - .\frac{1}{k^2} \\
\ge& \frac{\log 2}{\frac12 \log k} - \frac{1}{k^2} \\
\ge& 0
\end{align*}

Since $x_1 < 1/k$, we have
$$
\frac{2}{\log k} = kx_1e^{-kx_1} > \frac{kx_1}{e} \Rightarrow x_1 < \frac12.
$$

Thus,
$$
\frac{2}{\log k} = kx_1e^{-kx_1} > kx_1(1-kx_1) > \frac{kx_1}{2} \Rightarrow x_1 < \frac{4}{k\log k}
$$
and
\begin{align*}
f(x_1) =&  \frac{\log \frac{1}{x_1}}{\frac12 \log k} - e^{-kx_1} \\
\ge& \frac{\log \frac{k\log k}{4} }{ \frac12 \log k} - 1 \\
=& \frac{\log k + \log\log k - \log 4}{\frac12 \log k} - 1 \\
\ge& 0.
\end{align*}

Finally, (\ref{lemma_inequal:condition3}) is direct from the fact $e^x\ge x$.
\end{proof}

%% file: content/implementation.tex
\section{Additional Experiment Detail}
\subsection{Environment Settings}
\label{detail_env}

\subsubsection{Climb Game Variants}

(i) \textbf{One-step climb game}. A one-time climb game $G(n, k, u, U)$ is a $n$-player matrix game where every player has $U$ actions to choose from. The reward is determined by the number of players who choose action $u$, which can be defined as
\begin{align*}
    R(\boldsymbol{a}) = 
    \begin{cases}
        1,&\text{if \#}u = k,\\
        1 - \delta, &\text{if \#}u = 0,\\
        0,& \text{otherwise}.
    \end{cases}
\end{align*}
where $\delta\in(0,1)$.

The task space $\mathcal{T}^{\text{one}}_U$ contains all $2$-player one-step climb games with $U$ actions for each player, i.e., $\mathcal{T}^{\text{one}}_U = \{G(2,k,u,U)\mid 1\le k\le n, 0\le u < U \}$.

(ii) \textbf{Multi-stage climb game}. A multi-stage climb game \\ $\hat{G}(S, n, [(k_t,u_t)]_{t=1}^S, U)$ is an $S$-stage game, where each stage $t$ itself is a one-step climb game $G(n,k_t,u_t,U)$. At stage $t$, agent $i$ is given the observation $\mathcal{O}_i^t = [t, h_{t-1}]$, where $h_{t-1}$ is the history of the joint actions.

The task space $\mathcal{T}^{\text{multi}}_{S,U}$ consists of all $2$-player multi-stage climb game with $S$ stages and $U$ actions, i.e., $\mathcal{T}^{\text{multi}}_{S,U} = \hat{G}(S, 2, [(k_t,u_t)]_{t=1}^S,U) \mid \forall t\ 1\le k_t\le n, 0\le u_t< U\}$. We choose $\mathcal{T}^{\text{multi}}_{5,10}$ ($10$-Multi) for the experiments.

In all experiments $\delta$ is set to $\frac12$. We use $\mathcal{T}^{\text{one}}_{10}$ and $\mathcal{T}^{\text{multi}}_{5,10}$ in our experiments.The task distribution $p(\mathcal T)$ is uniform over the task space. Ten training tasks are sampled from the task distribution, and three testing tasks that are different from the training tasks are chosen to evaluate the performance.

\subsubsection{MPE Domain} \label{detail_mpe}

In a MPE Climb Game $\bar{G}(n,k,u,U,L)$, there are $U$ non-overlapping landmarks on the map with positions $\{L_j\}_{j=0}^{U-1}$. We assume a distribution $L\sim \Psi^U$ from which the landmark positions $L$. The reward is determined by the number of agents locating on the $u$-th landmark. More formally, suppose $f_j(s)$ is the number of agents locating on the $j$-th landmark, the reward can be defined as
\begin{align*}
    R(s, \boldsymbol{a}) = 
    \begin{cases}
        1,&\text{if } f_u(s)=k\text{ and }\sum_{j=0}^{U-1}f_j(s)=n,\\
        1 - \delta, &\text{if } f_u(s)=0\text{ and }\sum_{j=0}^{U-1}f_j(s)=n,\\
        0,& \text{otherwise}.
    \end{cases}
\end{align*}
The observation of agent $i$ contains the relative positions of all landmarks and other agents. As before, $u$ and $k$ will not be given in the observation and can only be inferred from the received reward.
A task space $\mathcal{T}^{\text{MPE}}_{n,U}=\{\bar{G}(n,k,u,U,L) \mid k=n,0\le 0<U, L\sim \Psi^U\}$ consists of all MPE climb games with $n$ players and $U$ landmarks, and is fully cooperative by setting $k=n$.

We evaluate MESA on the $2$-agent tasks ($\mathcal{T}^{\text{MPE}}_{2,5}$ and $\mathcal{T}^{\text{MPE}}_{2,6}$) and $3$-agent tasks ($\mathcal{T}^{\text{MPE}}_{3,5}$ and $\mathcal{T}^{\text{MPE}}_{3,6}$) while fixing $k=2$. The task distribution $p(\mathcal T)$ is defined by the probability density function $p(\bar{G}(n,k,u,U,L))=U^{-1}\Psi^U(L)$.

We set $\delta=\frac12$. The environments are different in landmark positions, and the tasks are different in target landmarks. 

\subsubsection{Multi-agent MuJoCo Domain}
In the multi-agent MuJoCo Swimmer environment, we similarly define a climb game $G(\alpha)$ where there are $2$ agents and the angular states of the two joints determine the current action of the two agents. Specifically, if state $s$ corresponds to the two joints forming angles of $\alpha_1$ and $\alpha_2$, then the reward can be defined as:
\begin{align*}
    R(s, \boldsymbol{a}) = 
    \begin{cases}
        1,&\text{if } |\alpha_1 - \alpha| < \epsilon, |\alpha_2 - \alpha| < \epsilon\\
        1 - \delta, &\text{if } |\alpha_1 - \alpha| > \epsilon, |\alpha_2 - \alpha| > \epsilon\\
        0,& \text{otherwise},
    \end{cases}
\end{align*}
where $\epsilon$ is a very small angle. The task space consists of target angles between $-30$ degrees and $30$ degrees, i.e., $\mathcal{T} = \{\alpha\mid -30^{\circ} < \alpha < 30^{\circ} \}$. We also set $\delta=\frac12$.

\subsection{Hyperparameter and Computation Settings} \label{hyper_sec}
The hyperparameters are detailed in Table \ref{table:hyperparam}. All tasks are sampled uniformly at random from the task space detailed in Section \ref{experiments} and then divided into the training and testing tasks. We use different tasks for the meta-training stage, which includes the high-reward dataset collection and the training of the exploration policies. We evaluate the meta-trained exploration policies on novel meta-testing tasks over $3$ runs with different seeds, each consisting of a different set of meta-testing tasks. Computation is done on a 32-core CPU with 256 GB RAM and an NVIDIA GeForce RTX 3090.

\begin{table*} [t]
\centering
\begin{tabular}{ll}
\hline
\textbf{Hyperparameter} & \textbf{Value} \\
\hline
off-policy MARL Algorithm $f$ & MADDPG \\
\hline
Meta-training steps & \begin{tabular}[c]{@{}l@{}} 50k in One-step Climb Game\\ 100k in Multi-stage Climb Game \\ 3M in MPE \\
3M in MA-MuJoCo
\end{tabular} \\
\hline
\begin{tabular}[c]{@{}l@{}} High-reward data collection steps
\end{tabular} & \begin{tabular}[c]{@{}l@{}} 30K in Climb Game\\ 500k in MPE \\
1M in MA-MuJoCo
\end{tabular} \\
\hline
Meta-training task size & 
\begin{tabular}[c]{@{}l@{}} $10$ in Climb Game \\
$30$ in MPE and MA-MuJoCo Swimmer
\end{tabular} \\
\hline
Meta-testing steps & \begin{tabular}[c]{@{}l@{}} 50k in One-step Climb Game\\ 100k in Multi-stage Climb Game \\ 3M/6M in MPE \\
2.5M in MA-MuJoCo Swimmer
\end{tabular} \\
\hline
Meta-testing task size & \begin{tabular}[c]{@{}l@{}} $5$ in Climb Game \\
$15, 18$ in $5$-agent MPE and $6$-agent MPE\\
$6$ in MA-MuJoCo Swimmer
\end{tabular} \\
\hline
Random exploration      & \begin{tabular}[c]{@{}l@{}} 3000 steps in Climb Game \\ 50K steps in MPE and MA-MuJoCo \end{tabular} \\
\hline
Network architecture   & \begin{tabular}[c]{@{}l@{}}Recurrent Neural Network\\ (one GRU layer with 64 hidden units) \end{tabular}   \\
\hline
Threshold $R^\star$   & 1 (sparse-reward tasks)   \\
\hline
Relabel $\gamma$   & 0.05  \\
\hline
Decreasing function $f_d$ & $1/x^5$ \\
\hline
Distance Metric $\Vert\cdot\Vert_\mathcal{F}$ & L2 norm \\
\hline
Number of Exploration Policies & \begin{tabular}[c]{@{}l@{}} 4 in Climb Game \\ 2/4 in MPE \\ 2 in MA-MuJoCo \end{tabular} \\ 
\hline
Learning rate & 5e-3/1e-4 (Adam optimizer) \\
\hline
Batch size & \begin{tabular}[c]{@{}l@{}} 32 trajectories in Climb Game \\ 300 trajectories in MPE \\ 8 trajectories in MA-MuJoCo \end{tabular} \\
\hline
\end{tabular}
\caption{Hyperparameters used in MESA}
\label{table:hyperparam}
\end{table*}

\subsection{Baseline Methods} \label{appendix:baseline}

For MAPPO \cite{MAPPO} and Random Network Distillation \cite{RND}, we use the released codebase\footnote{https://github.com/marlbenchmark/on-policy}.

For QMIX \cite{QMIX} and MAVEN \cite{MAVEN}, we use the released codebase\footnote{https://github.com/AnujMahajanOxf/MAVEN}. In MA-MuJoCo environments, we discretize the action into $11$ even points for QMIX.

For MAESN \cite{MAESN}, we modify the single-agent version to a multi-agent version by treating agents as independently optimizing their own returns without consideration of the other agents.

For EMC~\cite{peng2021facmac}, we use the released codebase\footnote{https://github.com/kikojay/EMC}.

For meta-training-based methods, we pretrain a policy (conditioned or unconditioned) with the same configuration of MESA (number of tasks, number of training steps) and then deploy it on the meta-testing task. The goal-conditioned policy takes the information of the task goal (key action for climb game, key landmark id for MPE Domain, key angles for MA-MuJoCo) as additional observation. We adapt C-BET \cite{parisi2021interestingdeepak} to multi-agent based on MAPPO \cite{MAPPO}.

\subsection{Results On All Environments}
\begin{table*}[t]
\centering
\begin{tabular}{|l|l|l|l|l|l|l|l|}
    \hline
    \multicolumn{1}{|l|}{}     & \multicolumn{1}{l|}{\begin{tabular}[c]{@{}l@{}}One-step\\ Climb Game\end{tabular}} & \multicolumn{1}{l|}{\begin{tabular}[c]{@{}l@{}}Multi-stage\\ Climb Game\end{tabular}} & \multicolumn{1}{l|}{\begin{tabular}[c]{@{}l@{}}2A5L MPE\end{tabular}} & \multicolumn{1}{l|}{\begin{tabular}[c]{@{}l@{}}2A6L MPE\end{tabular}} & \multicolumn{1}{l|}{\begin{tabular}[c]{@{}l@{}}3A5L MPE\end{tabular}} & \multicolumn{1}{l|}{\begin{tabular}[c]{@{}l@{}}3A6L MPE\end{tabular}} & \multicolumn{1}{l|}{\begin{tabular}[c]{@{}l@{}}MA-MuJoCo\\ Swimmer\end{tabular}} \\ \hline
    \multicolumn{1}{|l|}{MESA} & \multicolumn{1}{l|}{\textbf{0.83 $\pm$ 0.20}}                                                  & \multicolumn{1}{l|}{\textbf{0.81 $\pm$ 0.06}}                                                     & \multicolumn{1}{l|}{\textbf{61.32 $\pm$ 8.24}}                                                        & \multicolumn{1}{l|}{\textbf{58.73 $\pm$ 10.16}}                                                       & \multicolumn{1}{l|}{\textbf{51.83 $\pm$ 13.37}}                                                       & \multicolumn{1}{l|}{\textbf{44.71 $\pm$ 15.92}}                                                       & \multicolumn{1}{l|}{\textbf{599.32 $\pm$ 35.93}}                                             \\ \hline
    MADDPG                     & 0.74 $\pm$ 0.25                                                                       & 0.50 $\pm$ 0.00                                                                          & 21.09 $\pm$ 23.07                                                                            & 19.44 $\pm$ 18.41                                                                            & 5.45 $\pm$ 6.49                                                                              & 3.16 $\pm$ 5.73                                                                              & 499.33 $\pm$ 0.52                                                                   \\ \hline
    MAPPO                      & 0.50 $\pm$ 0.00                                                                       & 0.54 $\pm$ 0.05                                                                          & 24.42 $\pm$ 2.99                                                                             & 27.57 $\pm$ 4.86                                                                             & 10.12 $\pm$ 5.83                                                                             & 13.52 $\pm$ 2.91                                                                             & 496.88 $\pm$ 1.98                                                                   \\ \hline
    MAPPO-RND                  & 0.50 $\pm$ 0.00                                                                       & 0.53 $\pm$ 0.04                                                                          & 24.05 $\pm$ 2.45                                                                             & 27.10 $\pm$ 5.67                                                                             & 17.18 $\pm$ 3.52                                                                             & 7.30 $\pm$ 2.76                                                                              & 496.36 $\pm$ 1.76                                                                   \\ \hline
    QMIX                       & 0.50 $\pm$ 0.00                                                                       & 0.50 $\pm$ 0.00                                                                          & 0.21 $\pm$ 0.04                                                                              & 0.48 $\pm$ 0.12                                                                              & 0.06 $\pm$ 0.02                                                                              & 0.04 $\pm$ 0.03                                                                              & 499.97 $\pm$ 0.03                                                                   \\ \hline
    MAVEN                      & 0.50 $\pm$ 0.00                                                                       & 0.50 $\pm$ 0.00                                                                          & 0.06 $\pm$ 0.03                                                                              & 0.13 $\pm$ 0.10                                                                              & 0.06 $\pm$ 0.05                                                                              & 0.07 $\pm$ 0.00                                                                              & 495.41 $\pm$ 2.63                                                                   \\ \hline
    EMC                        & 0.50 $\pm$ 0.00                                                                       & 0.50 $\pm$ 0.00                                                                          & 50.49 $\pm$ 17.26                                                                            & 50.61 $\pm$ 17.41                                                                            & 36.63 $\pm$ 0.11                                                                             & 36.47 $\pm$ 0.28                                                                             & 499.66 $\pm$ 0.20                                                                   \\ \hline
    Pretrain                   & 0.55 $\pm$ 0.00                                                                       & 0.55 $\pm$ 0.00                                                                          & 23.26 $\pm$ 2.19                                                                             & 29.81 $\pm$ 4.97                                                                             & 17.91 $\pm$ 2.63                                                                             & 17.72 $\pm$ 5.28                                                                             & 496.49 $\pm$ 1.33                                                                   \\ \hline
    CBET                       & 0.55 $\pm$ 0.00                                                                       & 0.55 $\pm$ 0.00                                                                          & 23.69 $\pm$ 1.01                                                                             & 28.37 $\pm$ 3.59                                                                             & 19.85 $\pm$ 5.05                                                                             & 19.80 $\pm$ 4.36                                                                             & 497.54 $\pm$ 1.49                                                                   \\ \hline
    Goal                       & 0.55 $\pm$ 0.00                                                                       & 0.55 $\pm$ 0.00                                                                          & 31.61 $\pm$ 0.90                                                                             & 32.14 $\pm$ 0.10                                                                             & 30.60 $\pm$ 1.03                                                                             & 30.44 $\pm$ 0.60                                                                             & 498.13 $\pm$ 0.45                                                                   \\ \hline
    \end{tabular}
\caption{Summary of final performance for all the algorithms evaluated on all the environments. Numbers in bold indicate the best performing in a particular environment. For simplicity, 2A5L MPE stands for $2$-agent MPE with 5 landmarks and the same for other MPE results.}
\label{table:final_results}
\end{table*}

Table \ref{table:final_results} gives the final performance for each algorithm in all environments. We observe that our proposed MESA outperforms all other baseline methods across all environments.

\section{Visualization of Learned Exploration Policies} \label{appendix:visualization}

We visualize two exploration policies in the 2-agent 3-landmark MPE Climb Game tasks. Both exploration policies are shown in Figure \ref{fig:mode_two}.  In addition, the learned policy visited the three landmarks within $20$ timesteps, less than a third of the length of the trajectory, which showcases its ability to quickly cover the collected promising subspace. Both policies successfully visited all three landmarks consecutively and within only $1/3$ of the episode length. 

\begin{figure}[htb]
    \centering
    \includegraphics[width=.45\textwidth]{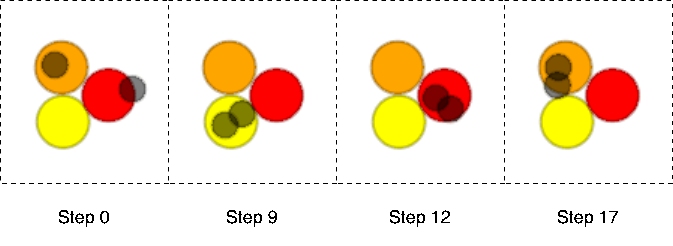}
    \caption{Visualization of two learned exploration policies in 2-agent 3-landmark MPE tasks. The figure shows the critical steps in the trajectory where agents coordinately reach high-rewarding states that were previously collected.}
    \label{fig:mode_two}
\end{figure}

%% file: content/ablation_datasetINIT.tex
\section{Ablation studies}
\label{appendix:ablation}

\begin{figure}
    \centering
    \includegraphics[width=.45\textwidth]{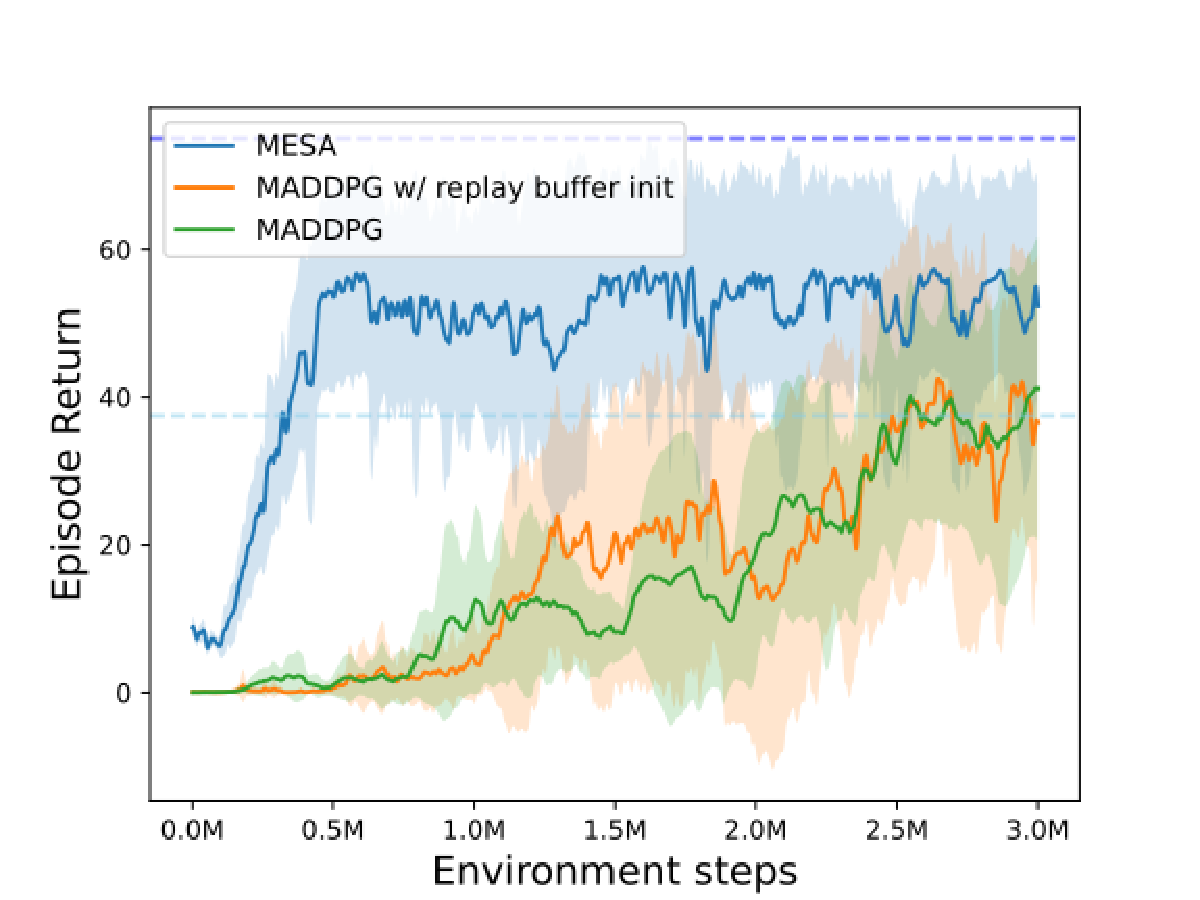}%
     \caption{Ablation study. Comparing our approach with 1) vanilla MADDPG without using the high-rewarding state-action dataset $\mathcal{M}_*$ or the trained exploration policies and 2) initializing the replay buffer with the $\mathcal{M}_*$ but not using the exploration policies, we show that both components contribute to overall performance.}
     \label{fig:ablation}
\end{figure}
To figure out the extent to which the high-rewarding state-action dataset $\mathcal{M}_*$ and the trained exploration policies contribute to the overall better performance, we perform an ablation study on the MPE domain.%

We observe in Figure \ref{fig:ablation} that by initializing the buffer with $\mathcal{M}_*$, the training process is accelerated (from $1$M steps to $2$M steps). However, even though $\mathcal{M}_*$ contains the collected high-reward states, directly learning the intrinsic structure and generalizing to unseen meta-testing tasks is nontrivial. Hence, MADDPG with $\mathcal{M}_*$ initialization still fails to learn the optimal NE.

But when assisted with the trained exploration policies, the algorithm is able to find the global optimum while also training faster than the vanilla MADDPG. The greater variance of the ablated method also suggests that $\mathcal{M}_*$ contains useful but subtle information for learning, and the exploration policies help with extracting that information.